\documentclass{article}

\usepackage[preprint]{neurips_2021}

\usepackage[utf8]{inputenc} % allow utf-8 input
\usepackage[T1]{fontenc}    % use 8-bit T1 fonts
\usepackage{hyperref}       % hyperlinks
\usepackage{url}            % simple URL typesetting
\usepackage{booktabs}       % professional-quality tables
\usepackage{amsfonts}       % blackboard math symbols
\usepackage{nicefrac}       % compact symbols for 1/2, etc.
\usepackage{microtype}      % microtypography

\usepackage{amsmath,amssymb,amsthm}
\usepackage[export]{adjustbox}

\newtheorem{proposition}{Proposition}[section]
\newtheorem{theorem}{Theorem}[section]
\newtheorem{corollary}{Corollary}[theorem]

\newtheorem{definition}{Definition}[section]
\newtheorem*{remark}{Remark}

\usepackage{mathtools}
\usepackage{commath}
\usepackage{bm}
\usepackage{graphicx}  % DO NOT CHANGE THIS
\usepackage{tikz}
\usetikzlibrary{bayesnet}
\usetikzlibrary{arrows}
\usepackage{color}
\usepackage[font={footnotesize=4pt,it}]{caption}
\usepackage{subcaption}
\usetikzlibrary{backgrounds}

\usepackage{neurips_2021}

\usepackage[utf8]{inputenc} % allow utf-8 input
\usepackage[T1]{fontenc}    % use 8-bit T1 fonts
\usepackage{hyperref}       % hyperlinks
\usepackage{url}            % simple URL typesetting
\usepackage{booktabs}       % professional-quality tables
\usepackage{amsfonts}       % blackboard math symbols
\usepackage{nicefrac}       % compact symbols for 1/2, etc.
\usepackage{microtype}      % microtypography
\usepackage{xcolor}         % colors

\title{Generalization by design: Shortcuts to Generalization in Deep Learning}

\author{%
  Petr Taborsky\\
  Technical University of Denmark\\
  Department of Applied Mathematics and Computer Science\\
  Richard Petersens Plads, 321, 2800 Kgs. Lyngby, Denmark\\
  \texttt{ptab@dtu.dk} \\
  \And
  Lars Kai Hansen \\
  Technical University of Denmark\\
  Department of Applied Mathematics and Computer Science\\
  Richard Petersens Plads, 321, 212, 2800 Kgs. Lyngby, Denmark\\
  \texttt{lkai@dtu.dk} \\
}

\begin{document}

\maketitle

\begin{abstract}
We take a geometrical viewpoint and present a unifying view on supervised deep learning with the Bregman divergence loss function - this entails frequent classification and prediction tasks. Motivated by simulations we suggest that there is principally no implicit bias of vanilla stochastic gradient descent training of deep models towards "simpler" functions. Instead, we show that good generalization may be instigated by bounded spectral products over layers leading to a novel geometric regularizer. It is revealed that in deep enough models such a regularizer enables both, extreme accuracy and generalization, to be reached. We associate popular regularization techniques like weight decay, drop out, batch normalization, and early stopping with this perspective. Backed up by theory we further demonstrate that "generalization by design" is practically possible and that good generalization may be encoded into the structure of the network. We design two such easy-to-use structural regularizers that insert an additional \textit{generalization layer} into a model architecture, one with a skip connection and another one with drop-out. We verify our theoretical results in experiments on various feedforward and convolutional architectures, including ResNets, and datasets (MNIST, CIFAR10, synthetic data). We believe this work opens up new avenues of research towards better generalizing architectures.
\end{abstract}

\section{Introduction and Contributions}
State of the art deep neural networks are trained with some form of regularization, including well-known tools such as weight decay, dropout, batch normalization \cite{ioffe2015batch} used to train deep ResNet architectures \cite{he2016deep} or more loosely regularization is achieved by "early stopping". Regularization is thought to keep weights 'under control' during training and thus positively impacts its generalization properties by reducing overfitting. Figure \eqref{fig1:weights_hist} suggests that the distribution of the weights indeed is related to how the model generalize on a previously unseen data, as argued by many \cite{huh2021low,jacot2018neural, kawaguchi2017generalization, goodfellow2016deep}.  We note that in most cases regularization is applied without analyzing how it may potentially interfere with the structure of the network.
\begin{figure*}[ht!]
    \centering
\begin{subfigure}[t]{.3\linewidth}
\centering
  \includegraphics[width=.95\linewidth]{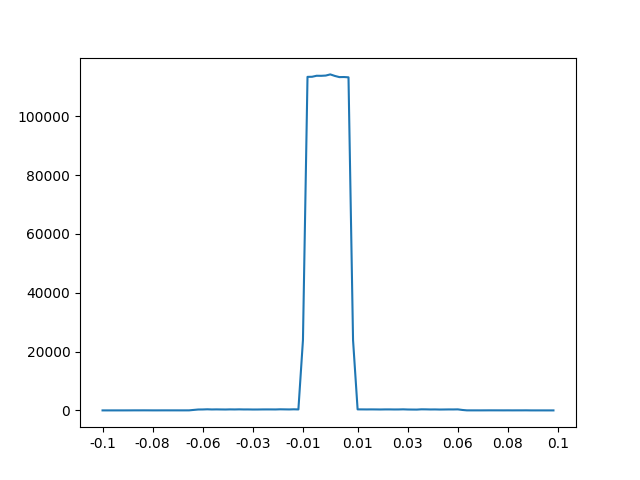}
   \caption{Initialization}
\end{subfigure}
\begin{subfigure}[t]{.3\linewidth}
\centering
  \includegraphics[width=.95\linewidth]{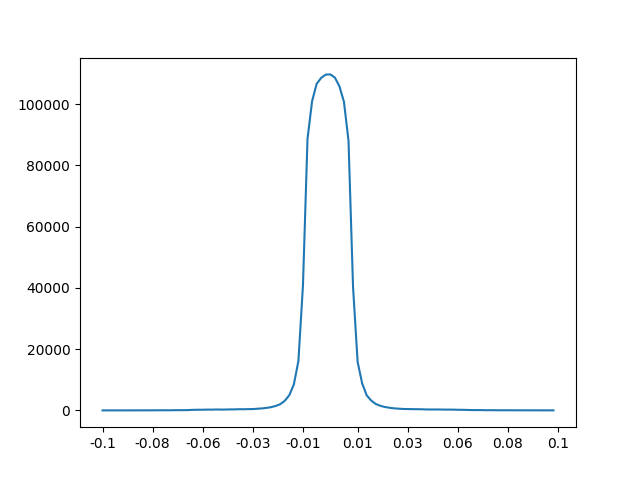}
   \caption{True labels}
\end{subfigure}
\begin{subfigure}[t]{.3\linewidth}
\centering
  \includegraphics[width=.95\linewidth]{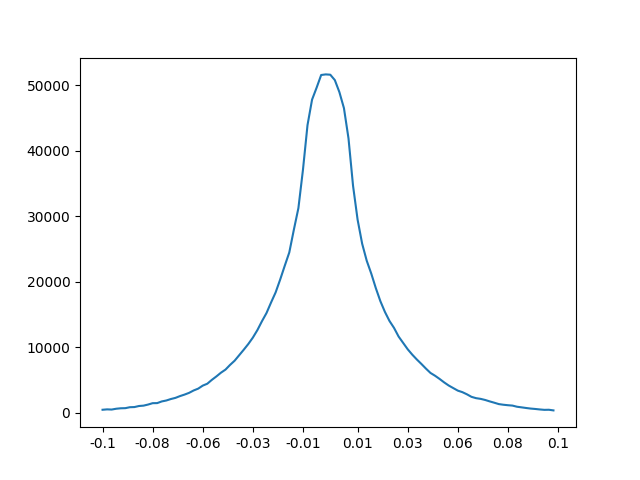}
   \caption{Random labels}
\end{subfigure}
\caption{\textbf{Histogram of weights of a convolutional network classifier of handwritten digits (MNIST) after zero training error has been reached}. Figure (a) depicts weights distribution after standard (He) initialization, while Fig.(b) shows a distribution of weights after reaching 98.6 test accuracy and Fig.(c) reports same but in case of fitting random labels as described in \cite{zhang2016understanding} (test accuracy corresponds to random choice. i.e. approx. 10\%). We can observe that a more complex model (c) has significantly heavier tails compared to the fully convergent and well generalizing close to optimum model (b) and that a distribution of weights, especially its tails, may carry an information on generalization properties.}
\label{fig1:weights_hist}
\end{figure*}

% a zero training error reached after approx. $100$ epochs%

Let alone the practical success and well developed theory, is omnipresent regularization really necessary in deep learning? There is a very wide and active area of research on an implicit regularization of deep learning with stochastic gradient descent, see our section on related work for more details. Besides many provable benefits, SGD alone seems to be not enough to ensure a good performance outside of training data however. In Fig.\ref{fig:overfitting} we show that three feed forward models of different depths, 1, 7 and 118 hidden layers, all overfit when trained by vanilla SGD for long enough time. So in the light of this "counter example" it seems some sort of regularization or "early stopping" is indeed needed. Here our aim is to improve network generalization by taking the architecture of the model into account and design a regularizing layer as an integral part of the model, we call it - \textit{ regularization by design}.

Without much loss of generality we consider loss functions of Bregman divergences in this paper. The choice of the Bregman divergence loss may seem limiting at the first glance but as shown in \cite{banerjee2005clustering} the Bregman divergence encompasses many common loss functions, such as squared loss, max. likelihood (KL-divergence), cross-entropy etc. used to train both predictive and classification  models. In section \ref{sec:theory} the Bregman divergence loss allows us to develop a geometrical perspective on generalization and characterize "well" generalizing networks in weight parameter space by using products of weights along the input-output paths, motivated by Fig.\ref{fig1:weights_hist} and aforementioned work.

We present 1.) a novel approach to generalization in neural networks - generalization by design. Following the theory presented in section three, we introduce a "generalization" layer in two forms. One encoded by skip connections the other using dropout. Each of them has specific training regime. As opposed to other regularization methods this generalizing layer is encoded into an architecture of a the model. We further design 2.) a practical method and demonstrate experimentally that this enhanced architecture significantly improves generalization, see Table.\ref{tab:ResNet_comparison}, and is applicable across classification and regression models in general. Finally, 3.) section three and discussion derives novel geometrical perspective, labelled Learning in the Manifold of Distributions, that establishes the arguments for a long puzzling problem of the extreme accuracy and generalization, that has often been reached by deep networks, observed but not yet fully understood \cite{zhang2016understanding}.

The paper is organized as follows. After this introduction, Section 2 lays out a geometrical perspective on SGD training and generalization of both classification and prediction neural networks, Learning in the Manifold of Distributions. Further, it describes our proposed "generalization by design" approach. Section 3 will cast the most commonly used regularization techniques, e.g., the dropout, batch normalization, early stopping, weight decay into the geometrical perspective of the previous section. Experiments on MNIST, CIFAR10, and synthetic data in Section 4 provide supporting evidence for the claim.

\section{Related Work}

Inspired by the work of \cite{hauser2018principles} and \cite{amari2016information} we aim to provide an intuitive and unifying geometric perspective on  generalization in deep networks. We investigate the long puzzling problem of why and how deep learning models generalize to unseen data so well despite their ability to fit arbitrary functions (random labels) \cite{zhang2016understanding}.

The work of \cite{hauser2018principles} uses a  differential geometric formalism on smooth manifolds to (re)define forward pass and back-propagation in a "coordinate" free manner. This allows us to see the layers of neural network as the "coordinate representations" of data (input) manifold $M$. As the number of nodes changes as one moves through the layers of the neural network we effectively change the dimensionality used by the neural network to represent the data manifold. The dimension of the underlying data manifold is the number of dimensions in the smallest "bottleneck" layer while all other layers are  immersion/embedding representations. For details we refer reader to \cite{hauser2018principles}.

Conveniently, the coordinates of an \textit{input} (data) layer are Cartesian and thus all data points are embedded into this smooth Riemannian manifold with $\ell_2$-norm induced inner product. A topology of this embedding space provides a well defined neighbourhood of all data points and allows us to analyze its behaviour\footnote{Note on a smoothness assumption on $f$}.

Overall \cite{hauser2018principles} works with an empirically supported assumption that neural network learns a sequence of coordinate transformations to put data into a flattened form, i.e., it is assumed that \textit{output} manifold is Riemannian and flat, that is the Riemannian metric tensor is constant.

Further a recent work \cite{he2020resnet, zhang2019towards} shows that identity connections, also called "shortcuts", besides counteracting vanishing gradients and degradation by depth, see. \cite{he2016deep}, help residual networks \cite{he2016deep, rousseau2018residual} to generalize well. Use of shortcuts of varying length motivated by modelling suitable representations of inputs are used in U-nets \cite{ronneberger2015u} and DenseNets \cite{huang2017densely} to improve an accuracy and convergence.

Many have investigated the dynamics of stochastic gradient descent (SGD), see e.g.  \cite{huh2021low, ali2020implicit,roberts2021sgd,smith2021origin, chizat2020implicit, volhejn2021does, kalimeris2019sgd,nakkiran2019sgd} and more, under the hypothesis that SDG is implicitly biased towards "simple" functions and as such regularizes itself. It provably finds minimum norm solutions in specific. i.e. convex, least squares optimization problems \cite{ali2020implicit}. Yet \cite{} shows that min. norm solution is not enough to guarantee good generalization in deep enough networks. Other line of works shows SGD to spend exponentially longer time in shallow minima \cite{xie2020diffusion,li2017stochastic} if approximated by Langevin dynamics with Gaussian noise. Bayesian and approximate inference based works \cite{smith2017bayesian} with many more interesting properties.

\begin{figure*}
    \centering
\begin{subfigure}[t]{.3\linewidth}
\centering
  \includegraphics[width=.95\linewidth]{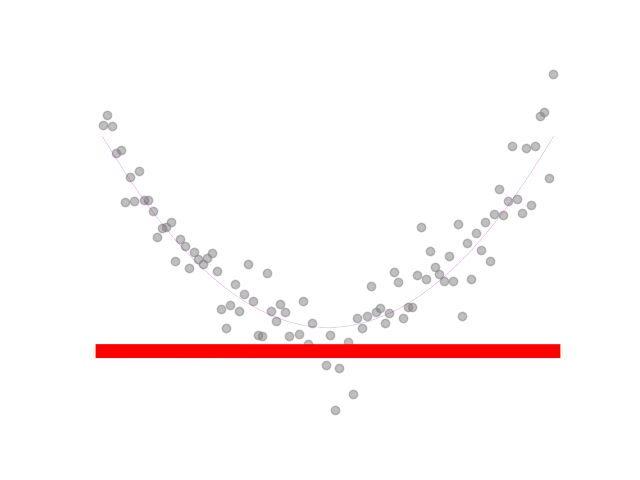}
   \caption{epoch 1}
\end{subfigure}
\begin{subfigure}[t]{.3\linewidth}
\centering
    \includegraphics[width=.95\linewidth]{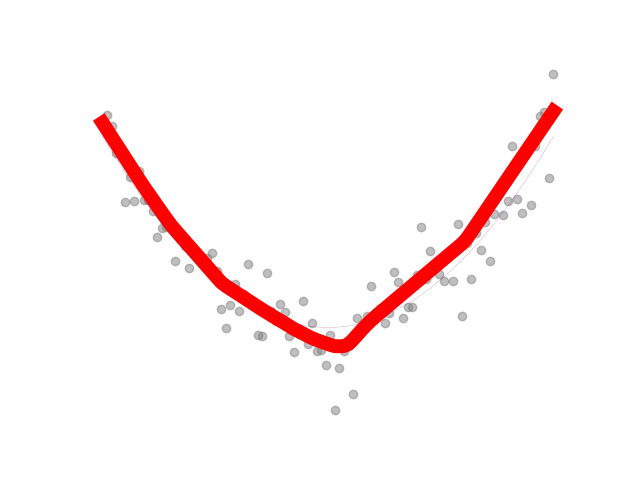}
    \caption{epoch 500}
\end{subfigure}
\begin{subfigure}[t]{.3\linewidth}
\centering
    \includegraphics[width=.95\linewidth]{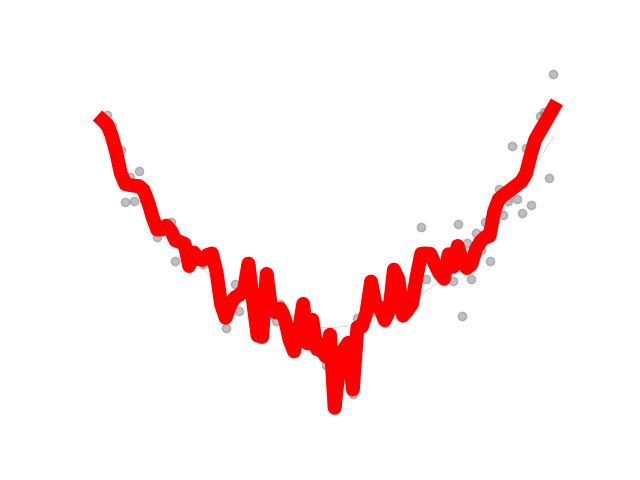}
    \caption{epoch 10,000}
\end{subfigure}
\begin{subfigure}[t]{.3\linewidth}
\centering
  \includegraphics[width=.95\linewidth]{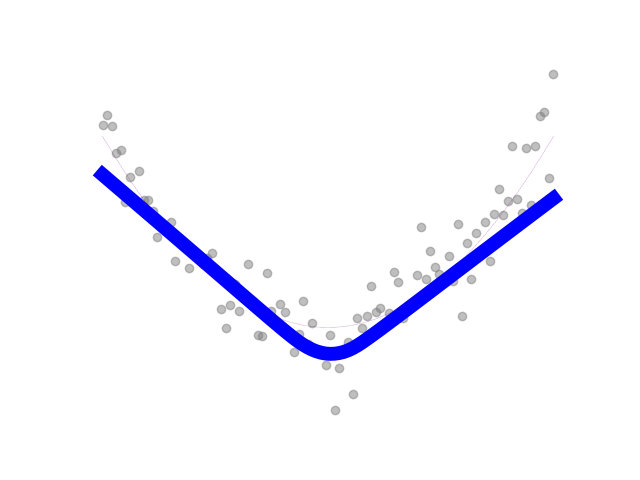}
   \caption{epoch 1}
\end{subfigure}
\begin{subfigure}[t]{.3\linewidth}
\centering
    \includegraphics[width=.95\linewidth]{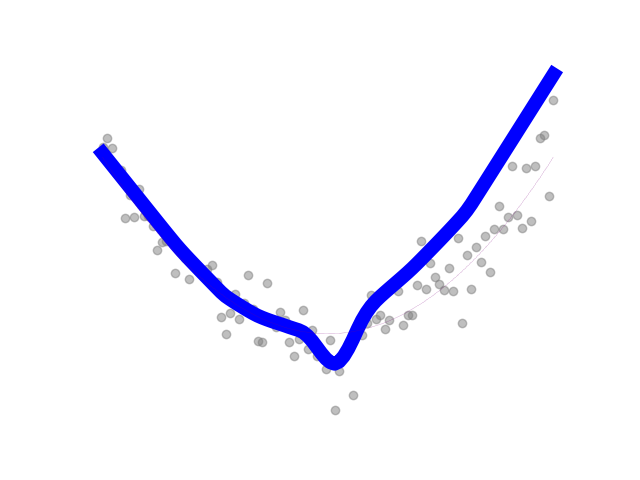}
    \caption{epoch 500}
\end{subfigure}
\begin{subfigure}[t]{.3\linewidth}
\centering
    \includegraphics[width=.95\linewidth]{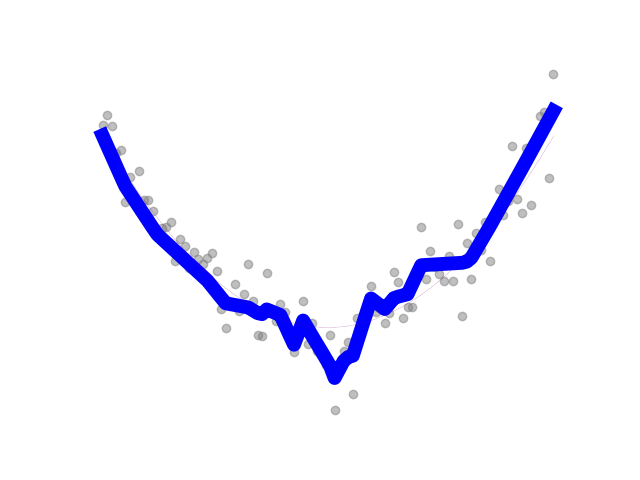}
    \caption{epoch 10,000}
\end{subfigure}
\begin{subfigure}[t]{.3\linewidth}
\centering
  \includegraphics[width=.95\linewidth]{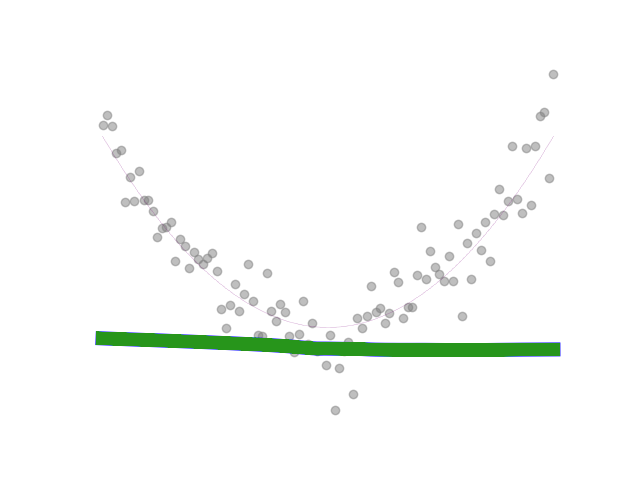}
   \caption{epoch 1}
\end{subfigure}
\begin{subfigure}[t]{.3\linewidth}
\centering
    \includegraphics[width=.95\linewidth]{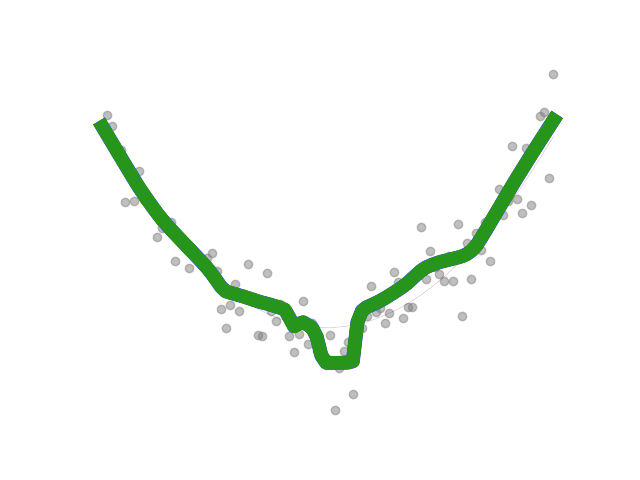}
    \caption{epoch 500}
\end{subfigure}
\begin{subfigure}[t]{.3\linewidth}
\centering
    \includegraphics[width=.95\linewidth]{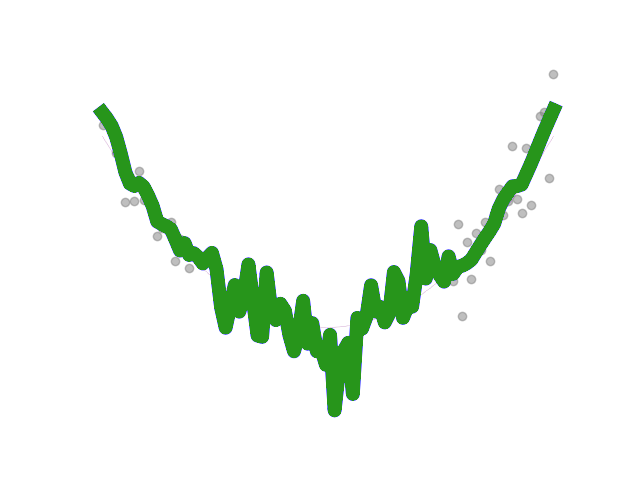}
    \caption{epoch 10,000}
\end{subfigure}
\caption{\textbf{Does vanilla SGD implicitly self-regularize?} Fitting a noisy quadratic function: $g(x)=10 + 1.22x^2 + \epsilon, \epsilon \sim N(0,10^{2})$ by a feed forward network (FFN) with 1 hidden layer (blue) and 7 hidden layers (red) and residual FFN with 119 hidden layers and skip connections (green). All models use ReLu activations and have comparable number of parameters, i.e., 35,989 (blue) and 35,998 (red) and 36160 (green). They are trained by vanilla SGD for over 10,000 epochs with a constant learning rate of $\eta=10^{-5}$ and no regularization or batch normalization in case of ResNet. We can observe that independent of the depth or architecture of the models they all overfit a training data eventually, as reported in the right most column. A training data consists of randomly generated values of $g(x)$ on 100 data point grid on the interval (-8,8) (gray), test data similarly on interval (-2,2).}
\label{fig:overfitting}
\end{figure*}

\section{Learning in the Manifold of Distributions} \label{sec:theory}

Following the approach outlined in the Introduction section and using the work of \cite{hauser2018principles} as a basic concept developed furhter in this section we see layers of the neural network model as a coordinate representations of a data manifold $M$. The neural network $f$ is than a composition of layer-to-layer maps between these coordinate representations of $M$ and parametrized by the learnable weights. Altogether it forms by assumption a smooth\footnote{ReLu can be seen as a limit case of smooth NN models using softmax activations, see comments in the end of the section} input-output map. 

The work of \cite{zhang2016understanding} showed that flatness of the weight space is insufficient to capture generalization in neural networks. This section rethinks a generalization of neural network $f$ by first, endowing the output layer with dually flat Riemannian geometry and then, second, pulling its metric back to an input layer along $f$ making use of differential geometry. 

Having established link between known metric spaces\footnote{If loss function is Bregman divergence than output layer is dually flat Riemannian manifold with Fisher information matrix as a metric tensor, see \cite{amari2016information}} we define a local generalization of network $f$ around every training datum $\bm{x}$ by "flatness" of this pulled-back metric (a bilinear form at given $\bm{x}$ defined by its symmetric matrix) and weight compositions of $f$, formalized in the main Proposition 3.1. of this section. In the last step, we derive global generalization from local one by applying local properties on every training datum and deriving global sufficient conditions for inducing local generalization (Corollary \ref{corollary_3}).

As noted in \cite{hauser2018principles} an input layer is a rather arbitrary representation of a data manifold $M$. For example an RGB representation of the image used for image recognition is used just because RGB is a convenient format for image software and displays. However it is not a good representation for labeling/image recognition task. 

On contrary an output layer is by definition and by our choice of loss function, i.e., Bregman divergence, a regular probabilistic model best suited for a given modelling task. For example for a binary classifier with cross entropy loss an output layer is a regular one dimensional \footnote{dually flat} parameter manifold corresponding to Binomial distribution (of Exponential family) with a cumulant function $\psi(\theta)=-\log(\theta)$ and spanned by natural parameter $\theta=\log(\frac{p_0}{1-p_0})$, where $p_0$ defines probability of a particular class, e.g. '0'. From this modelling perspective the data manifold is best and intuitively represented by the output layer.

Further following \cite{hauser2018principles} and using machinery of differential geometry that provides a necessary formal justification a neural network can be seen in reverse as embedding/immersing an output (dually) flat\footnote{see Supplementary material and \cite{amari2016information} for details} layer into an arbitrary input layer. This embedding/immersion is being updated over the course of training driven by optimization in the output layer and (back)propagating the output layer gradient towards the input layers. In this perspective a choice of Bregmann loss defines a probabilistic Exponential family model in the output layer whose number of linearly independent parameters is upper bounded by dimensionality of the output.

Following the binary classification example above, an output layer model is a one-dimensional unimodal convex model parameterized by $\theta$ or by transformed $p_0$. In the setting of deep learning\footnote{as well as in linear regression or Gaussian processes} the neural network model $f$ maps data points $\bm{x}_s$, with $s$ indexing training data samples, to the point $\theta_s$ in the output manifold of Binomial distributions - that is for every pair of input target $(\bm{x}_s, \bm{y}_s)$ it is a different input dependent distribution being fitted. In other words, a map $f$ fits input dependent metric tensor to the geometry of output layer model of targets\footnote{see the Appendix for more on dually coupled Exponential Families} $\bm{y}_s)$. Every $f(\bm{x}_s,\bm{w})$ gives rise to an inner product, that is in case of induced Exponential family the Fisher Information Matrix, $FIM(\theta_s)\mid_{\theta_s=f(\bm{x}_s,\bm{w})}$, of the output layer manifold that depends on $s$ and thus differs across data points in general. For every $s$ a mode (in this case ML estimate of $p_0$) of the output layer model maps to a multitude of weight space optima \cite{amari2016information}, section 12.2. They all would have the same level of loss, however.

In just outlined perspective a learning proceeds in this manifold of distributions and the overall generalization of the model $f$ is determined by generalization properties of all training data-point dependent models as opposed to standard one probability model view. 

Further we follow the machinery of differential geometry and use the neural network $f$ to "pull" the output layer (data dependent) inner product FIM($f(\bm{x}_s,\bm{w})$) "back" to the input layer (a so called "pull-back metric" \cite{hauser2018principles}), see Fig.10 in Supplementary Material (Appendix). 

This operation "imprints" specific properties of $f$ into this "pulled back inner product" which is symmetric bilinear positive \textit{semi}definite form, $\langle \cdot , \cdot \rangle_{f(\bm{x}_s)}$ defined in the neighbourhood of $(\bm{x}_s)$ and characterized by its matrix, see \cite{gantmakher1959theory, bhatia1997a}. When we use "flatness" of bilinear form we have in mind the operator norm of the matrix of this bilinear form. We'll abuse notation for brevity throughout the paper.

Notably while a flatness of this bilinear form at $\bm{x}_s$, or rather its related quadratic form $\langle \bm{x} , \bm{x} \rangle_{f(\bm{x}_s)}$ for $\bm{x}$ in the neighbourhood of $(\bm{x}_s)$, captures the generalization of the map $f$ naturally using the curvature of pulled back inner product (and its eigenvalues, see Fig.10 in Appendix), it differs essentially from a flatness of $f$ around $\bm{x}_s$. While the pulled back metric at $\bm{x}_s$ depends on the first derivatives of $f$ (Jacobian), given by backpropagation or chain rule of derivative (see Supplementary Material, Fig. 1 and Eq.(4) "Back-propagated Inner Product of the Manifold of Distributions") the flatness of $f$ would be defined by the second derivatives of $f$ (Hessian).

In short and fundamentally, it is the Jacobian of $f$, involved in this quadratic form, not the Hessian, that defines generalization in this paper.

On an example of the ReLu network, Hessian gives a constant flat structure ($f$ is piece-wise linear) while generalization, defined by the flatness of pulled back metric above, depends on products of weights that are activated at particular data input, $\bm{x}_s$, promoting weights configurations resulting in "small" products. This idea underpins Proposition 3.1 \ref{theorem:generalization_theorem}, one of the main results of the paper.

Following the idea we use "flatness" of the quadratic form $\langle \bm{x} , \bm{x} \rangle_{f(\bm{x}_s)}$ seen as a function of $\bm{x}$ in the neighbourhood of $\bm{x}_s$ to define local generalization properties of the model. A global generalization arises from the local one applied "for every $\bm{x}_s$".

In particular, getting back to \cite{zhang2016understanding}, fitting random labels viewed from the perspective above would correspond to fitting the individual models around training data points exactly, yet pulled back inner products would have large curvatures, given by randomness in labels of neighboring training data and thus resulting in bad generalization properties. On contrary, if the generalization of the kind above was enforced (by methods developed later on in the paper) the randomly initialized model would not converge to solution fitting random labels. In other words, upon "good generalization" constraints on weights, the weight configuration fitting random labels becomes unreachable.

This "rethought" generalization is a basis for our results that will materialize in Proposition \ref{theorem:generalization_theorem} capturing the idea of local generalization. 

Along the lines above the Corollary \ref{corollary_3} then finds an upper bound of the weight products that will be used to design a \textit{generalization layer} structure that would instigate the flatness of related pulled back metric around \textit{every} training data point and thus enforce global generalization of the model.

\subsection*{Generalization in the Manifold of Distributions}

Following the previous section a generalization on the output layer at datum $x_s$ is given by spectral properties of the Fisher information matrix (FIM) at $f(x_s)$ of the chosen probability model over targets $y_s$. The flatter the landscape w.r.t. outputs $f(x)$ playing the role of natural parameters, see Appendix \ref{sec:appendix_bregman}, the better. Nevertheless we'd like to know how the log likelihood changes with regards to an input layer and how does it depend on the model parameters. 

From now on we use upper indices to denote coordinates while lower one index vectors/tensors. Also the Einstein summation is used whenever pair of indexes appears in the equation to simplify notation. Let's further denote $g_{i,j}(\xi)=\langle \rm{e}_i, \rm{e}_j \rangle$ an inner product on the input space. It is a constant identity matrix in our case, i.e. $g_{i,j}(\xi)=\delta_{i,j}$. 

Without loss of generalization consider $f$ being a smooth $\mathcal{C}_2$ non-invertible map $f:I \xrightarrow{} F$ between input manifold $I$ with a coordinate system denoted $\xi^i$ and output manifold $F$ with a coordinate system $\theta^{\kappa}$ \footnote{As such it defines push-forward operator $f^*$ that acts on tangent spaces: $f^*:TI \xrightarrow{} TF$ and it can be viewed as a generalized coordinate free derivative, see. \cite{hauser2018principles} and Supplementary Material}.

Because an input space is by our choice of input data representation the Euclidean space with an ortho-normal basis it collides with its tangent space. If we for a sake of brevity take for granted that a back-propagation is well defined and is formally equivalent to a coordinate free chain rule over manifolds, for details see \cite{hauser2018principles}, we can skip a formalism of a differential geometry and use an ordinary derivatives and chain rule instead of a directional derivatives in manifolds and a right Lie-group actions (defined using "pull-backs") on a frame bundles.

Having in mind this simplification (of retracting to formalism of linear algebra instead of differential geometry one) an infinitesimally small line element $d\theta_{\kappa}$ in output (tangent) space $TF$ relates to an input $dx^i$ vector by Jacobian $J^{\kappa}$
\begin{align}
    d\theta_{\kappa}=J^{\kappa}_i dx^i \label{eq:small_line_element}
\end{align}
where Einstein summation over index $i$ is used.

Similarly for an output layer (a probabilistic manifold of probability distributions defined by a choice of Bregmann loss function corresponding to a cumulant function $\psi$) we have its metric tensor defined as 
\begin{align}
    g_{\kappa,\lambda}(\theta)=E[\partial_{\kappa}\log p(\tilde{\bm{y}},\theta) \partial_{\lambda}\log p(\tilde{\bm{y}},\theta)] \text{ (FIM) }=\partial_{\kappa}\partial_{\lambda}\psi(\theta) \label{eq:FIM_as_hessian}
\end{align}
where $\tilde{\bm{y}}=\nabla \varphi(\bm{y})$ is a random variable of the output layer Exponential family distribution derived from a dual Bregman divergence, $\varphi$ and $\psi$ being convex conjugates, for details see Supplementary material and \cite{banerjee2005clustering} and where the second equation is a consequence of a dual flatness, \cite{amari2016information}. 

Because $f$ is not (invertible) coordinate transformation, to relate distances of input and output coordinate representations properly is not trivial and it is done formally in \cite{hauser2018principles}. However, formally involved an idea follows a geometric intuition as follows. 
We leverage an induced geometry of the output layer which is given by our choice of the loss, i.e. probabilistic model of the output layer. Note that by constraining ourselves to use of Bregman divergence as a loss\footnote{that covers common cost functions, such as log-likelihood (KL divergence), cross-entropy, $\norm{}_2$, etc., used for both classification and prediction models, see \cite{banerjee2005clustering}} both inner products are known and thus $f$, as a smooth map by assumption, relates distances between input and output (tangent) spaces. We emphasize that this combination of input and output geometry is a crucial element that enables us to advance previous works on the topic.
Omitting rigorous definitions for brevity (see Appendix \ref{sec:appendix_proof} for more) we follow an intuitive notion of a "generalization" of $f$ around datum $x$: the larger the vicinity of given input point $x$ is mapped into a fixed output space 'around' $f(x)$, i.e., a unit hyper-ball, the better function $f$ generalizes in its vicinity. A formal build-up, definitions and proof, and additional explanatory figures are deferred to the accompanying Supplementary material in Appendix \ref{sec:appendix_proof}.

\begin{proposition}[Informal]\label{theorem:generalization_theorem}
In the context of the above, assuming activation functions used in architecture of $f$ are 1-Lipschitz a pull-back metric from an output layer into an input Euclidean manifold around a datum $x$ is up to a constant defined by the following positive semidefinite matrix:
\begin{align}
    \zeta(x) &= P(x) {\zeta_{\psi}(\bm{x})} P^T(x) \label{eq:prop31_zeta}
\end{align} 
where $P(x)$ is a real matrix $\{p_{i,j}\}$ with elements:
\begin{align}
    p_{i,j} &:= \sum\limits_{p \in \mathcal{BP}} {}^i_j\pi_p L_p(x), \label{col_2:pathproducts}
\end{align}
$\zeta_{\psi}(x):=\partial_{\lambda}\partial_{\lambda}\psi(f(\bm{x}))$ is fully determined by the chosen Bregman divergence loss. Further $L_p(x)$ is a positive real function formed by products of activation functions' derivatives along the path $p$ such that $0 \leq L_p(x) \leq 1$ and where $\mathcal{BP}$ is a set of all "back-propagation" paths connecting any input layer node to an output layer node through the network $f$ such that each layer has exactly one node present in the path. Then ${}^i_j\pi_p = \prod\limits_{\{l:w_l \in p\}} w_l$ is a product of all weights from an input node $i$ to an output node $j$ along the path $p$.
\end{proposition}

If we follow an idea of generalization of $f$ being described as a "flatness" of a positive semidefinite bilinear form\footnote{or rather the operator norm of its matrix but we'll keep abusing notation as noted before}, around a datum $x$, then according to the preceding Proposition \eqref{theorem:generalization_theorem}, $\zeta(x)$ can be used to asses the degree of this flatness by looking at the eigenvalues of the positive definite matrix $\zeta(x)$. Next corollary shows that the degree of flatness is upper bounded by the products of the largest eigenvalues of $\zeta(x)$.

\begin{corollary}[Spectral products, Informal] \label{corollary_3} Under conditions of Proposition \ref{theorem:generalization_theorem} the largest eigenvalue of $\zeta(x)$ from Eq.\ref{eq:prop31_zeta} can be bounded from above by a following product of eigenvalues:
\begin{align}
    \sigma_{\psi}(x) C\prod\limits_{l \in L_{f}} \sigma^2_{l} , \text{ and $C$ is a positive constant} \label{col_2:pathproducts_upperb}
\end{align}
where $L_{f}$ denotes set of layers of $f$ and $\sigma_{l}$ denotes the largest eigenvalue of matrix $W^{(l)}$ comprising the weights of the layer $l$ and similarly $\sigma_{\psi}(x)$ denotes the largest eigenvalue of the positive semidefinite outer layer metric tensor $\zeta(x)$.
\end{corollary}

\begin{proof}
By definition of $\zeta(x)$, an equivalence of norms on finite vector spaces and a 1-Lipschitz property of the activation functions by assumption. Full proof is provided in the Supplementaty material to be found in Appendix \ref{sec:appendix_proof}.
\end{proof}

Proposition \ref{theorem:generalization_theorem} and corollary \ref{corollary_3} relate a geometric notion of the generalization as a "flatness" to a parametric space of weights. In particular, it conveys that functions $f$ defined by a model architecture, activation functions, loss function, in line with \cite{kawaguchi2017generalization}, and a particular point in the weight space generalize well at point $x$ if sums of products of the largest eigenvalues of layer weight matrices are small. Despite a tighter bound can be given using traces\footnote{follows also from norm equivalence on finite vector spaces} or sum of weight products along input-output paths, for a development of our method in next section this bound suffices.

\iffalse
As outlined before the usefullness of the corollary follows from $C\prod\limits_{l \in L_{f}} \sigma^2_{l}$ being independent on trainng data point and thus allowing for a control of the global ("for every $\bm{x}_s$") generalization properties. 
\fi

\begin{remark}[Towards Global Generalization \label{remark:global_generalization_in_manifold_of_distributions}] 
Notably the upper bound \eqref{col_2:pathproducts_upperb} of Corollary \ref{corollary_3} takes a form of product and the $\prod\limits_{l \in L_{f}} \sigma^2_{l}$ part does not depend on input data $x$. It is given by structure of $f$, i.e. composition of layers, activation functions and weights yet it is "global". If this part is kept low and $\sigma_{\psi}(x)$ is bounded on $\mathcal{D}$\footnote{it is for example in case of a regular model of the output layer and a finite dataset} the largest eigenvalue of $\zeta(x)$ (curvature around $x$) would be small for all $x$. This is the driving idea for the design of \textit{generalization layer} in the next section. 

Strikingly, it follows that under flatness constraint on every distribution indexed by $s$, overfitting, i.e., reaching zero training error, is not a problem as it could be in the case of no such regularizer. On contrary under imposed constraints on "flatness" of the pulled back bilinear form as introduced above (Proposition \ref{theorem:generalization_theorem} and \ref{app:corollary_3}) reaching zero training error improves accuracy and is demanded. On the other hand, reaching it requires necessary capacity e.g. depth of the model. 

In the extreme case of shallow 1 hidden layer feed-forward network, i.e. an input layer is followed by an output layer $\prod\limits_{l \in L_{f}} \sigma^2_{l} \approx 0$ means the networks $f$ is close to a constant function no matter how wide it is. That is it has extremely limited capacity. Notably going deeper $\prod\limits_{l \in L_{f}} \sigma^2_{l} \approx 0$ can be achieved by keeping eigenvalues of some layers small and that enables the remaining layers to be more expressive because they act in the product with previous layers. The same way the generalization layer is designed and experimentally verified to work in the following sections. 

Hence the depth enables both to reach zero training error by increasing the capacity of the network as well as to keep all pulled back metrics around training data flat, which is generalize well. In the explicit experiment (motivated differently though) it is demonstrated in \cite{hauser2018principles}, Fig.4.7. 

Nevertheless, the learning in the manifold of distributions suggests the explanation of a long puzzling phenomenon: the extreme accuracy and generalization the deep neural networks are able to reach at the same time. Moreover, it renders depth being the essential enabler of this capability.
\end{remark}

\subsection{Generalization Layer (GL) with Skip Connections} \label{sec:generalization_layer}
Corollary \ref{corollary_3} of the previous section suggests that good generalization can be achieved by keeping products of maximum eigenvalues of layers "small". In this section we design one such widely applicable method as later on shown in experiment section \ref{sec:experiments}.

Idea is to add an additional structure into the existing architecture as depicted in Fig.\ref{fig:gen_layer} with hooks to control size of weights (blue lines in Fig.\ref{fig:gen_layer}) during the training. 

This additional structure called \textit{generalization layer} is defined by the tuple: (nodes ($\bm{g}$), weights ($W_g$) and skip connection ($\nu W_s$) parametrized by a scalar $\nu$). Number of new nodes $\bm{g}$ is the same as $\bm{x}(l)$ to match dimensionality of weight matrix of layer $l$. Similarly skip connections $W_s$ are chosen to match the dimensionality of $x^{(l+1)}$. In the simplifying diagram Fig.\ref{fig:gen_layer} skip connections are of the form ($\nu I$), where $I$ represents the identify matrix because $x^{l}$ and $x^{(l+1)}$ have the same number of nodes. In general case however, when $x^{l}$ and $x^{(l+1)}$ are of different dimensionality, $W_s$ defines a linear projection matrix in line with notation used in \cite{he2016deep}. 

\begin{figure*}[tbh]
    \centering
\begin{subfigure}[t]{.45\linewidth}
\centering
  \includegraphics[width=.54\linewidth]{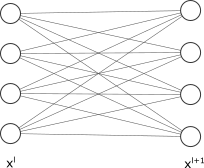}
  \caption{before}
\end{subfigure}
\begin{subfigure}[t]{.45\linewidth}
\centering
  \includegraphics[width=.95\linewidth]{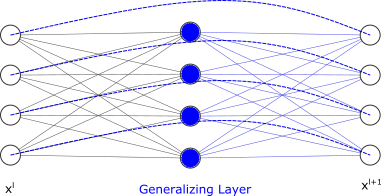}
    \caption{after}
\end{subfigure}
\caption{\textbf{Generalization layer:} Depicting the insertion of the "generalization layer" (blue) in (b) between layer $(l)$ and $(l+1)$. This structure comprises additional (blue) nodes of the same size as $x^{l}$, additional weights (blue lines) and skip connections (dashed blue lines) ensuring gradient flow is not broken into the original architecture (a).}
\label{fig:gen_layer}
\end{figure*}

This construction straightforwardly generalize to convolutional or any other architecture.

\textbf{Keeping the weight products under control} To deliver a desired regularizing effect it is necessary to prevent newly added weights $W_g$ from growth while ensuring the rest of the network is trained. This could be achieved by many means.  

We have opted for training with 10 times lower learning rate applied $W_g$ compared to the rest of the network in addition to steering large gradients towards the lower layer through shortcuts $W_s$ controlled by a hyperparameter $\nu$. In particular constraining the weight of the generalization layer dampens the backpropagated gradients during training and may stall or stop training whatsoever. To avoid breaking the gradient flow the skip connections are parametrized by scalar $\nu \in (0.1, 1)$ and this hyperparameter is linearly decayed during the training. For details see section \ref{sec:training_gen_layer}.

The method just described delivers the desired effect and the resulting model performs on par or better than original ResNet models \cite{he2016deep} despite in our design we took out the whole second residual block of stack convolutional layers in ResNet and replaced it by one FF layer as in Fig.\ref{fig:gen_layer} and thus our model has significantly less parameters. Results are presented in Experimental section \ref{sec:experiments}

\begin{figure*}
    \centering
\begin{subfigure}[t]{.23\linewidth}
\centering
    \includegraphics[width=.95\linewidth]{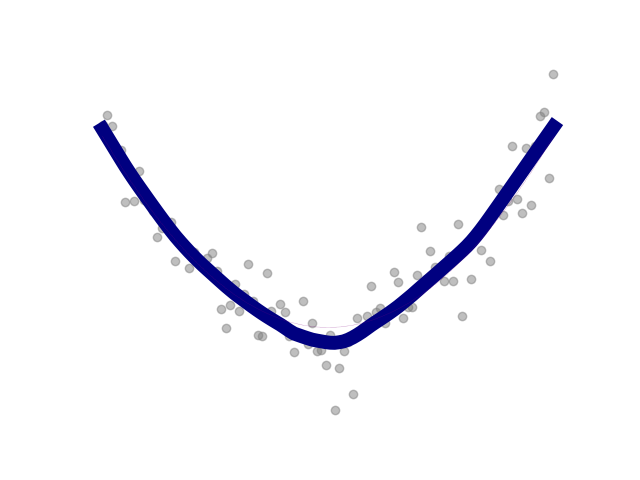}
    \caption{epoch 2000}
\end{subfigure}
\begin{subfigure}[t]{.23\linewidth}
\centering
  \includegraphics[width=.95\linewidth]{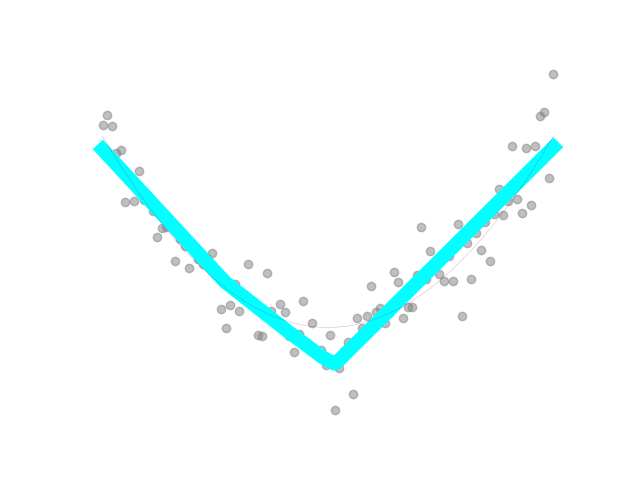}
   \caption{epoch 2000}
\end{subfigure}
\begin{subfigure}[t]{.23\linewidth}
\centering
    \includegraphics[width=.95\linewidth]{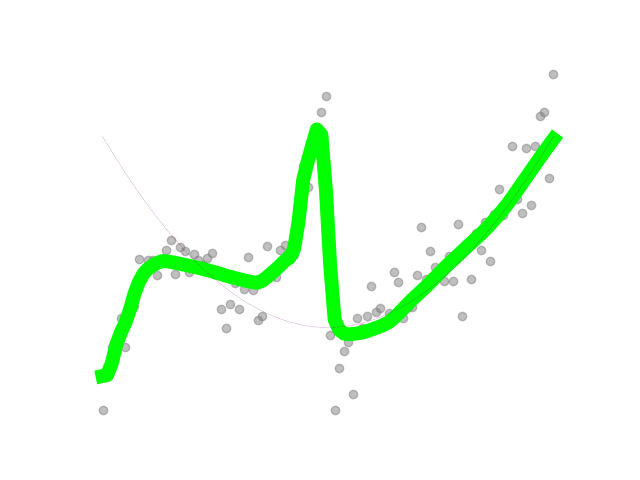}
    \caption{epoch 500}
\end{subfigure}
\begin{subfigure}[t]{.23\linewidth}
\centering
  \includegraphics[width=.95\linewidth]{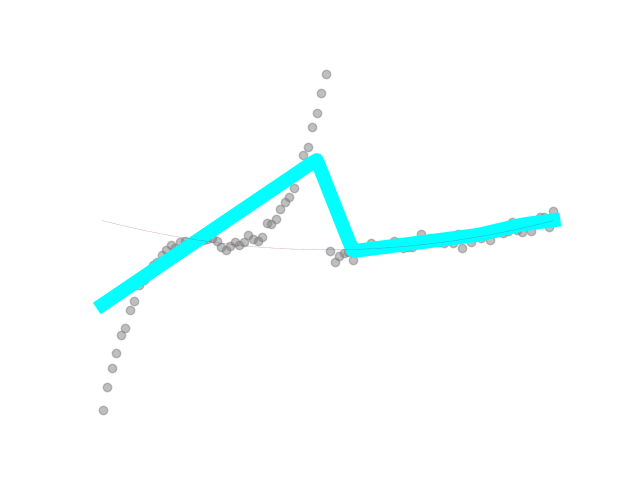}
   \caption{epoch 500}
\end{subfigure}

\begin{subfigure}[t]{.23\linewidth}
\centering
    \includegraphics[width=.95\linewidth]{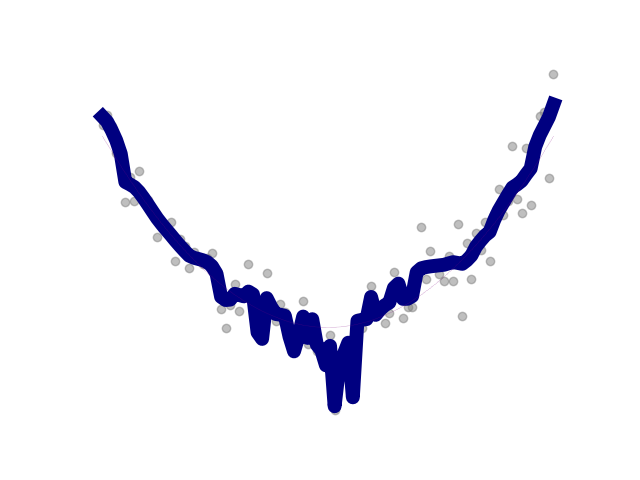}
    \caption{epoch 20,000}
\end{subfigure}
\begin{subfigure}[t]{.23\linewidth}
\centering
  \includegraphics[width=.95\linewidth]{img/resnet_longshortcuts/RESNET_NOISE_10_1e5_LONGSKIP_epoch_2000_train_data.png}
   \caption{epoch 20,000}
\end{subfigure}
\begin{subfigure}[t]{.23\linewidth}
\centering
    \includegraphics[width=.95\linewidth]{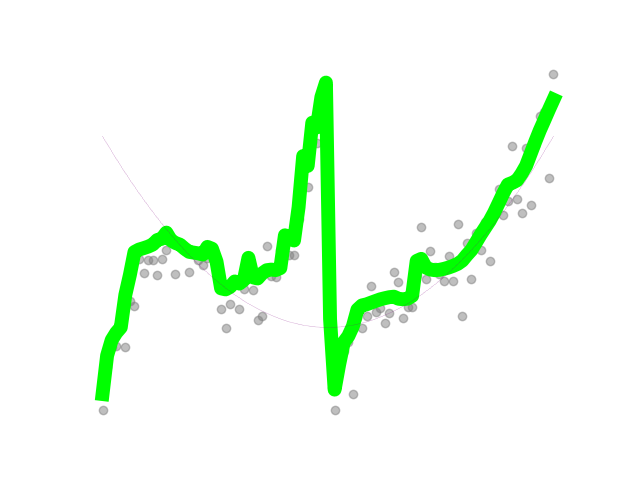}
    \caption{epoch 20,000}
\end{subfigure}
\begin{subfigure}[t]{.23\linewidth}
\centering
  \includegraphics[width=.95\linewidth]{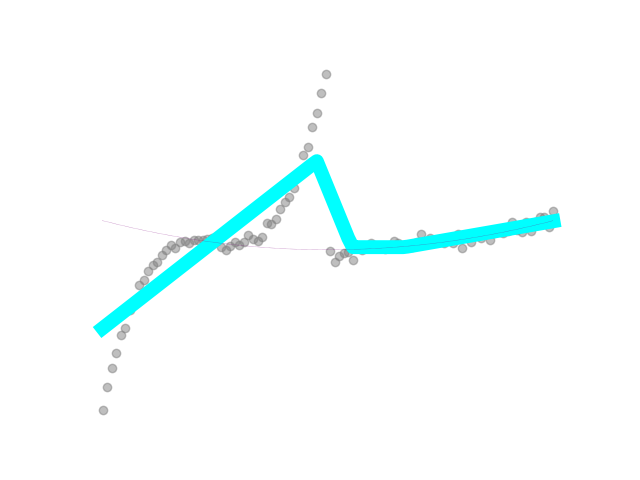}
   \caption{epoch 20,000}
\end{subfigure}

\caption{\textbf{Regularization by skip connections:} Least square regression of noisy quadratic function in the first two columns and mixture of quadratic and cubic function in the two rightmost columns. Presented are selected epochs during SGD training of the residual feed-forward architecture with ReLu activations, 119 layers and 36,160 trainable weights regularized by (long) skip connections over every $7^{th}$ hidden layer as described in Experimental setup \ref{sec:experimental_setup}. While the (cyan) model in the second and the fourth column is trained by vanilla SGD, the first column (purle) and third column (green) is the result of SGD training using the adaptive learning rate (Adam). Figure demonstrates that vanilla SGD in combination with long skips regularizer effectively enforces piecewise linear structure on fitted function from early epochs (b) and (d) to over 20 000 epochs in both cases, quadratic function (f) as well as in case of more complex the third degree noisy polynom in (h). This is in stark contrast to an adaptive training by Adam (not reported but similar results were produced by RMSProp as well) that explores flatter directions in loss landscape \cite{goodfellow2016deep} and recovers almost exact generating function, as shown in (a) and (c), however if there is no early stopping applied it eventually overfits, in (e) and (g), despite the regularization of skip connections (same as in vanilla SGD case). A training data consists of randomly generated values of $g(x)$ on 100 data point grid on the interval (-8,8) (gray), test data similarly on interval (-2,2).}
\label{fig1:regularization_by_longskips}
\end{figure*}

\subsection{Generalization Layer with DropOut} \label{sec:gld}

Previous section derived one particular method called Generalization Layer (GL), making use of the shortcut connections was designed to allow experimental verification of the theory. However, as noted therein the GL is not "the only" method. On the contrary we believe many are to be discovered.

In this section, we design another easy-to-implement generalization layer that works well in practice as will shortly be demonstrated in experiments with the CIFAR10 dataset, see \ref{sec:experiments_gld}.

Same as before in the case of GL the idea is to add structure into the existing architecture as depicted in Fig.\ref{fig:gld_design} with hooks to control the size of weights (blue lines in Fig.\ref{fig:gld_design}) during the training. 

This additional structure labeled "GLD" is defined by the tuple: (nodes $\bm{g}$, weights ($W_g$) and a hyperparameter 
$p \in [0,1]$). Dropout hyperparameter $p$ defines a probability of the success in Bernoulli trials of removing the node of the GLD layer (only) from the computational graph for one particular batch. Applied independently on all nodes of the GLD layer. Importantly a \textbf{dropout} is applied only on the GLD layer and independently on other regularization techniques used for training, incl. dropout.

The number of new nodes $\bm{g}$ is the same as $\bm{x}(l)$ to match the dimensionality of the weight matrix of layer $l$. In the simplifying diagram Fig.\ref{fig:gld_design} $x^{l}$ and $x^{(l+1)}$ have the same number of nodes. In general case however, when $x^{l}$ and $x^{(l+1)}$ are of different dimensionalities $W_g$, i.e. number of outgoing links, colorcoded blue in Fig.\ref{fig:gld_design}, is adjusted accordingly. 

Once embedded into an architecture of the desired model this extended model is trained by standard (stochastic) gradient descent (backpropagation) techniques with a caveat of an additional application of "GLD" dropout on this layer.

\begin{figure*}[tbh]
    \centering
\begin{subfigure}[t]{.45\linewidth}
\centering
  \includegraphics[width=.55\linewidth]{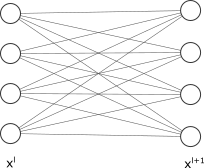}
  \caption{before embedding GLD}
\end{subfigure}
\begin{subfigure}[t]{.5\linewidth}
\centering
  \includegraphics[width=.95\linewidth]{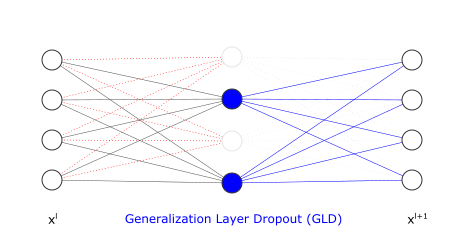}
    \caption{the extended architecture after}
\end{subfigure}
\caption{\textbf{Generalization Layer with Dropout (GLD):} Depicting the insertion of the "generalization layer" (blue) in (b) between layer $(l)$ and $(l+1)$. This structure comprises additional (blue) nodes of the same size as $x^{l}$, additional weights (blue lines) and probability of dropping added nodes of GLD layer from training (forward-backward) step. For instructive purposes a figures simulate the case when a half of the GLD nodes were ruled out for an step.}
\label{fig:gld_design}
\end{figure*}

This construction straightforwardly generalizes to convolutional or any other architecture and includes any additional structure that fits incoming and outgoing dimensionality. 

On the other hand, because Corollary \ref{corollary_3} is based on the upper bound that may be loose, there are differences in effectiveness across different versions of \textit{generalization layer}, as shown in Experiments i.e. GL vs. GLD in this paper. We believe many other forms of "generalization layer(s)" are to be explored as Proposition 3.1. only requires to control (ideally but not necessarily) all backpropagation paths $\mathcal{BP}$ no matter architecture or the way it is delivered.
\section{Experiments} \label{sec:experiments}
Our experiments include both feed forward and convolutional architectures and show the following:
\begin{itemize}
    \item Fig.\ref{fig:overfitting} demonstrates on the example that vanilla SGD alone eventually converges to a highly complex function that overfits training data in accordance to capacity of the model and suggests that an implicit regularization of SGD is not enough to achieve good generalization on its own.
    \item Fig.\ref{fig1:regularization_by_longskips} and Fig.\ref{fig:power_of_gl} report that it is possible to keep model in a "simple function" mode over long training periods by manipulating an architecture of the models. In this case we use additional skip connections over 7 layers to regularize as outlined in \cite{zhang2019towards}. It is of utmost importance to state that short skip connections without use of batch normalization readily overfit as demonstrated in Fig.\ref{fig:overfitting}.
    \item Other hereby not reported experiments we ran suggest the regularization by GL if overdone may keep a map $f$ "too simple" preventing SGD from improving training error for long periods. Further research on $\nu$ hyperparameter optimization an on alternative training regimes e.g. cyclical learning rates \cite{smith2017cyclical} is suggested.
    \item On classification task on the CIFAR10 Fig.\ref{fig2:regularization_by_longskips} demonstrate that GL replacing the middle stack of resnet blocks, see \cite{he2016deep}, and still achieving on-par or better performance while using less parameters.
    \item Further Fig.\ref{fig:gld_performance_1} and Fig.\ref{fig:gld_performance_thebestof} report experiments with another version of generalization layer with dropout instead of skip connections, this time placed beyond the encoder block of ResNet. Outstanding results of this encoding suggest that placement of the generalization layer within the mode architecture plays an essential role, see discussion.
    \item Table\ref{tab:ResNet_comparison} reports on popular ResNet architectures and the CIFAR10 image classification dataset, the 56 layers ResNet model enhanced by GLD structural regularizer outperformed all other original ResNets, \cite{he2016deep}, including the one with 1202 layers.
\end{itemize}

\subsection{Experimental Setup and "Generalization Layer" Training Regime} \label{sec:experimental_setup}
In experiments reported in Fig.\ref{fig:overfitting} and Fig.\ref{fig1:regularization_by_longskips} we use the least square regression of two noisy polynomial functions:
\begin{itemize}
    \item $g(x)=10 + 1.22x^2 + \epsilon, \epsilon \sim N(0,10)$
    \item $h(x)=10 + 1.22x^2 + 1.22(x+4)^3 + \epsilon, x \in (-\infty, 0) \epsilon \sim N(0,10)$ and $h(x)=g(x)$ otherwise (on $x \in (0,\infty)$).
\end{itemize}

For regression on $g(x)$ in Fig.\ref{fig:overfitting} a feed forward architecture of varying depth is used in counter example of SGD bias towards simple functions. All models are of feed forward (FF) architecture. The deepest (labeled "Res Net w/o BN" in the figure) has 119 hidden linear layers $17$ neurons each and one dimensional input and output so that number of trainable parameters (36,160) is as close as possible to the 7 hidden layers (Deep) model with 35,989 parameters and 1 hidden layer (Shallow) model with 35,998 trainable parameters for comparison purposes. All models use ReLu activations and are trained by SGD with a constant learning rate of $\eta=10^{-5}$ and no regularization (referred to as "vanilla SGD") unless stated otherwise, over 20,000 epochs. 
Synthetic training dataset for regression consists of randomly generated data points of $g(x)$ and $h(x)$ in (Fig.\ref{fig1:regularization_by_longskips}) on 100 data points grid on the interval (-8,8) reported, test data similarly on interval (-2,2).

In Fig.\ref{fig1:regularization_by_longskips} identity skip connections a.k.a. "shortcuts", see \cite{he2016deep, zhang2019towards}, are used to explore their regularization capability. In particular we use skips to shortcut every $7^{th}$ feed forward layer by identity and we refer to this model in Fig.\ref{fig1:regularization_by_longskips} as regularized by skip connections. We train the model with both vanilla SGD and Adam to showcase the affect of training with adaptive learning rate (Adam) with results elaborated on in the caption.

Fig.\ref{fig:power_of_gl} reports on the functional fit of $h(x)$ experiment by making use of multiple\footnote{3 blocks of FFL-ReLu-FFL with overarching $\nu$-weighted shortcut over each. Altogether has 90 hidden layers with 45 skip connections (out of which 3 skips belong to GL). Width of the layer is 17 units.} "GL" layers it is possible to recover the generating cubic function the way adaptive learning rate (Adam) SGD did (the best MSE achieved in the experiments), see Fig.4(a) and Fig.4(c). We explicitly note that results of Fig.\ref{fig:power_of_gl} have been achieved with GL as the only explicit regularizer used, i.e. no batch normalization, weight decay or else, using the training method described in Section 3 of the paper.

\subsubsection{Training with a "Generalization Layer"}\label{sec:training_gen_layer}
Experimental design for Fig.\ref{fig1:regularization_by_longskips} uses "regularizing layer" as designed in \ref{sec:generalization_layer} together with a scheduled weight decay applied on the skip connection parameter over the course of training. This regime is necessary to ensure gradient flow is not disconnected by "generalization" layer whose weights we'd like to keep low. To propagate gradient beyond this layer skip connections are used. Their strength $\nu$ is toned down during the training starting from $\nu=1$ and linearly decayed down to $\nu=0.1$ for the last 20\% of the training period\footnote{bringing $\nu=0$ is not wanted as it would disconnect smoothness of coordinate transformations over layers, as noted in \cite{zhang2019towards,hauser2018principles}}.

On top of previous 10 times lower learning rates of generalization layer weights were used after initial 20 epochs to slow down the growth of the largest eigenvalue of this layer in line with \ref{corollary_3}.

All experiments were designed and coded in PyTorch \cite{paszke2017automatic} and executed on regular 10 GPU cluster.

\subsection{Results of Experiments}
\subsubsection*{Experiments on Noisy Polynomial Functions}
Most of the results are elaborated on in the captions of figures Fig. \ref{fig:overfitting} and Fig.\ref{fig1:regularization_by_longskips}. In relation to "generalization by design" method specifically Fig.\ref{fig1:regularization_by_longskips} demonstrates a usefull regularization effect of skip connections, as covered in \cite{zhang2019towards}. As described in the caption of the figure shortcuts in connection with vanilla SGD training produces piece-wise linear fuction even after 20000 epochs. However if used in an adaptive learning rate SGD training regime it recovers generating functions perfectly yet without early stopping it continues to lower training error and overfit eventualy.

To shed more light on regularization effect of skip connection it is important to compare Fig.\ref{fig:overfitting} to \ref{fig1:regularization_by_longskips}. Fig.\ref{fig:overfitting} shows that ResNet architecture with short skips every second layer trained by vanilla SGD, i.e., without batch normalization\footnote{that besides reducing a covariate shift also regularizes, see \cite{ioffe2015batch}}, (labeled "Res Net w/o BN" in the figure) tends to heavily overfit from an early stage of the training. This is taken into account when designing "regularizing layer", where skips connections are used to steer gradient flow away from "bottleneck" layer in early stage of training rather than for their regularizing effect itself. See section \ref{sec:generalization_layer} for details.

\subsubsection*{Skip Connections GL On the CIFAR10 Experiments}
Experiments on CIFAR10 dataset are reported to demonstrate applicability and effectiveness of a "regularization layer" from previous section \ref{sec:generalization_layer} on the real world dataset and popular ResNet architecture. For results see caption of Fig.\ref{fig2:regularization_by_longskips}.
\begin{figure*}[htb]
    \centering
    \begin{subfigure}[t]{.85\linewidth}
    \centering
      \includegraphics[width=0.95\linewidth]{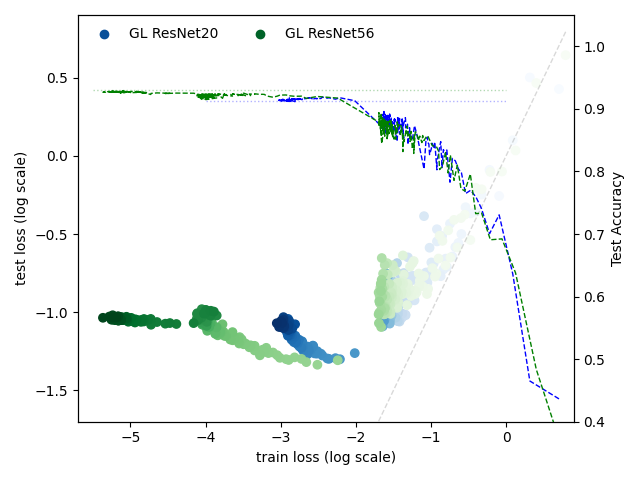}
    \end{subfigure}
\caption{\textbf{Generalizing layer (GL) on CIFAR10:} Both architectures GL ResNet20 (blue) and GL ResNet56 (green) based on original ResNet architectures referenced in the labels have fewer parameters than original models because they replaced the whole second ResNet convolutional block by one generalizing layer with skip connections defined in \ref{sec:generalization_layer}. Yet they reach on-par or better results than reported in the original paper \cite{he2016deep} - color-coded dotted lines indicate reference accuracy levels reached therein. A color-coding of epochs goes from lighter early ones to darker later ones.}
\label{fig2:regularization_by_longskips}
\end{figure*}

\begin{figure*}[htb!]
    \centering
\begin{subfigure}[t]{.3\linewidth}
\centering
  \includegraphics[width=.95\linewidth]{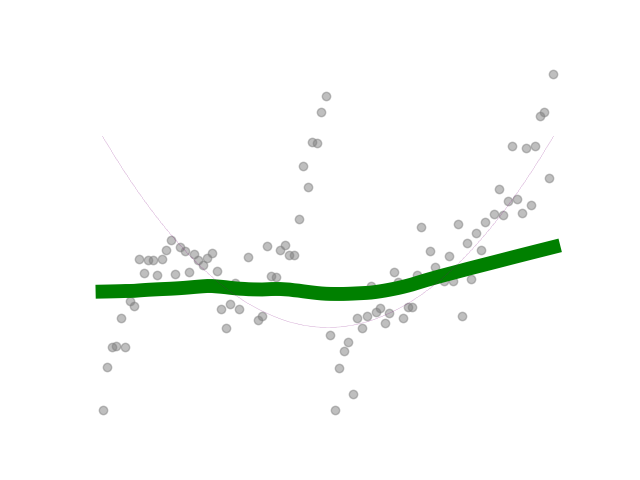}
   \caption{epoch 1}
\end{subfigure}
\begin{subfigure}[t]{.3\linewidth}
\centering
    \includegraphics[width=.95\linewidth]{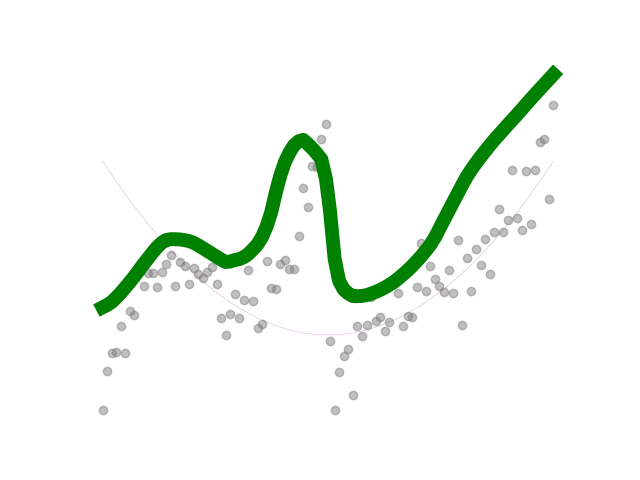}
    \caption{epoch 30}
\end{subfigure}
\begin{subfigure}[t]{.3\linewidth}
\centering
    \includegraphics[width=.95\linewidth]{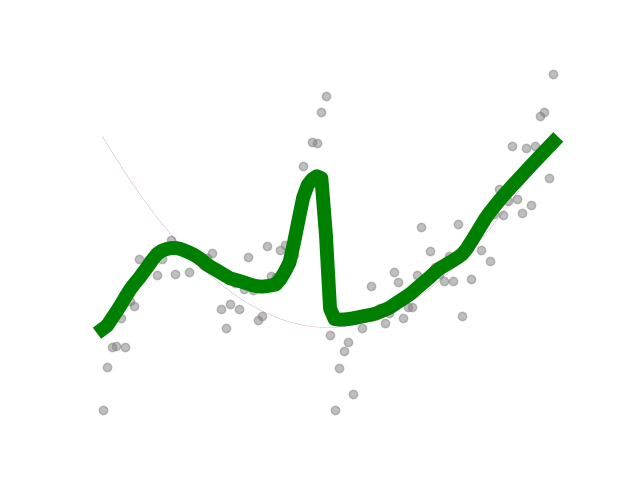}
    \caption{epoch 50}
\end{subfigure}
\begin{subfigure}[t]{.3\linewidth}
\centering
  \includegraphics[width=.95\linewidth]{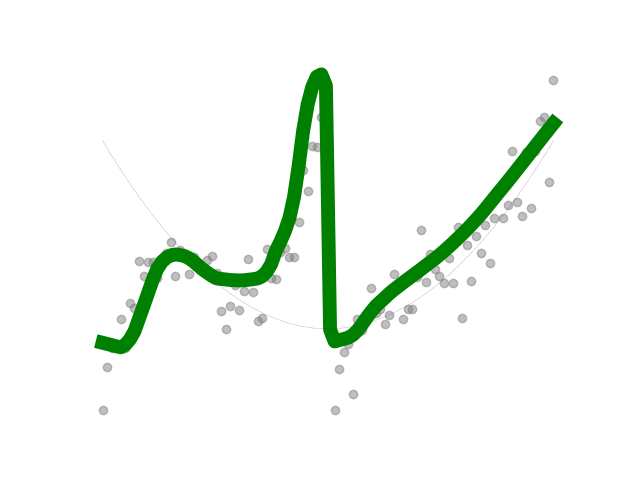}
   \caption{epoch 100}
\end{subfigure}
\begin{subfigure}[t]{.3\linewidth}
\centering
    \includegraphics[width=.95\linewidth]{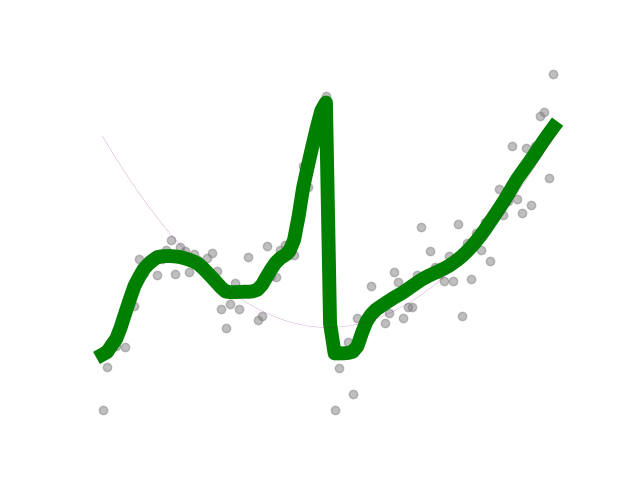}
    \caption{epoch 200}
\end{subfigure}
\begin{subfigure}[t]{.3\linewidth}
\centering
    \includegraphics[width=.95\linewidth]{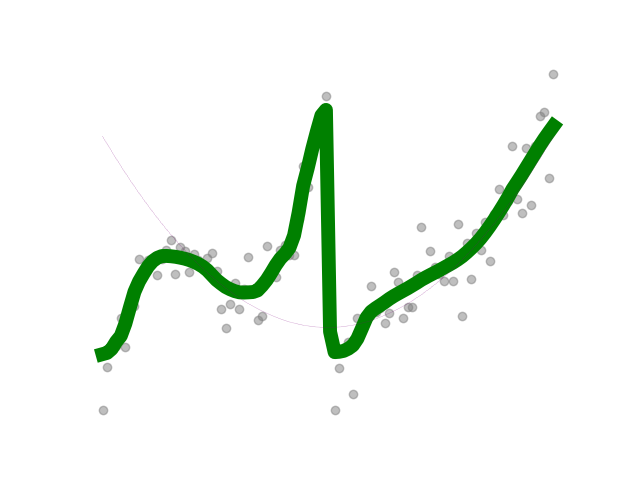}
    \caption{epoch 300}
\end{subfigure}
\begin{subfigure}[t]{.3\linewidth}
\centering
  \includegraphics[width=.95\linewidth]{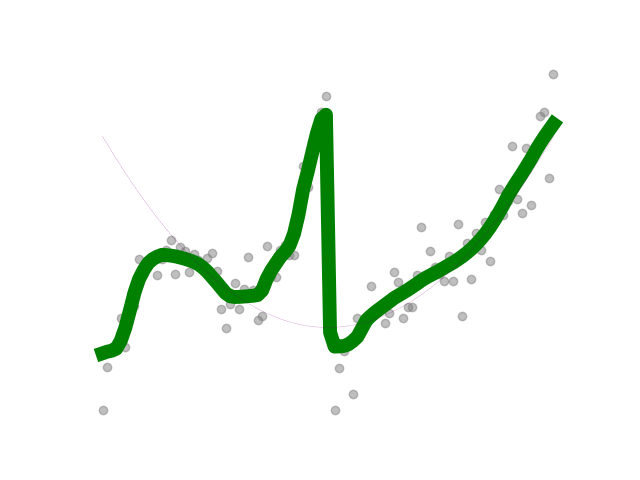}
   \caption{epoch 400}
\end{subfigure}
\begin{subfigure}[t]{.3\linewidth}
\centering
    \includegraphics[width=.95\linewidth]{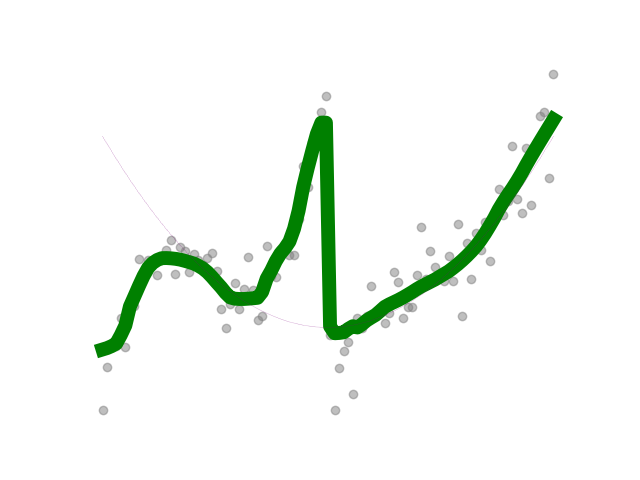}
    \caption{epoch 500}
\end{subfigure}
\begin{subfigure}[t]{.3\linewidth}
\centering
    \includegraphics[width=.95\linewidth]{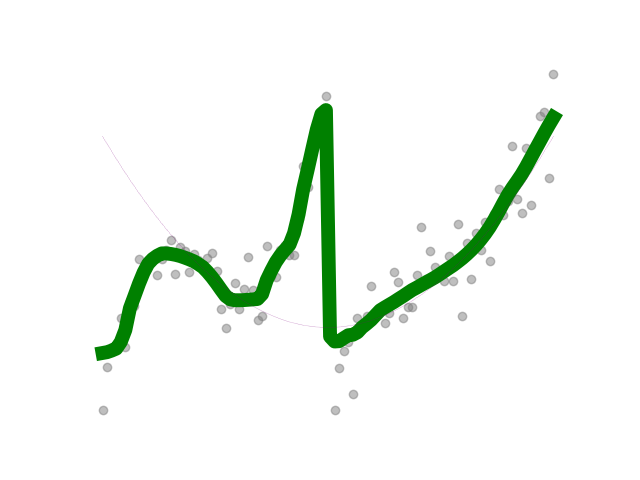}
    \caption{epoch 1,000}
\end{subfigure}
\caption{\textbf{Generalization layer (GL) as the only regularizer} Fitting a noisy cubic function: $h(x)=10 + 1.22x^2 + 1.22(x+4)^3 + \epsilon, x \in (-\infty, 0)$ and $h(x)=10 + 1.22x^2 + \epsilon$ otherwise (on $x \in (0,\infty)$), where $\epsilon \sim N(0,10)$. As shown in the snapshots from training (a)-(i) a generating function is recovered comparably to Fig4.(c) result of adaptive learning rate (Adam).
The only explicit regularizer used is the "generalization layer" (GL) of 3 blocks of FFL-ReLu-FFL with a $\nu$ weighted shortcut overarching every one of the three blocks as described in the paper Section 3.
The model consists of 90 hidden feed-forward layers (FFL) of 17 nodes wide with 45 skip connections (out of which three skips belong to GL weighted by hyperparameter $\nu$ and the rest are identities). The models were trained by vanilla SGD for over 1,000 epochs with a learning rate of $\eta=10^{-7}$ decayed every 200 epochs by 0.1. Further, GL was trained according to the regime from Section 3 with a learning rate $0.1\eta$ and $\nu$ linearly decayed to value 0.1 at epoch 500 and stayed on. A training data consists of randomly generated values of $h(x)$ on 100 data point grid on the interval (-8,8) (gray)}
\label{fig:power_of_gl}
\end{figure*}

The results of this alternative implementation demonstrate that improved generalization is achieved irrespective of the method of implementing Proposition 3.1. and its Corollary and as such is not a consequence of the method but rather of the concept presented.

\subsubsection*{Drop-out Generalization Layer (GLD) On the CIFAR10 Experiments} \label{sec:experiments_gld}
Experiments on the CIFAR10 dataset demonstrate applicability and effectiveness of a "regularization layer" from previous section \ref{sec:gld} on the real world dataset and popular ResNet architecture. For results see caption of Fig.\ref{fig:gld_performance_thebestof}.
\begin{figure*}[htb]
    \centering
    \begin{subfigure}[t]{.85\linewidth}
    \centering
      \includegraphics[width=0.95\linewidth]{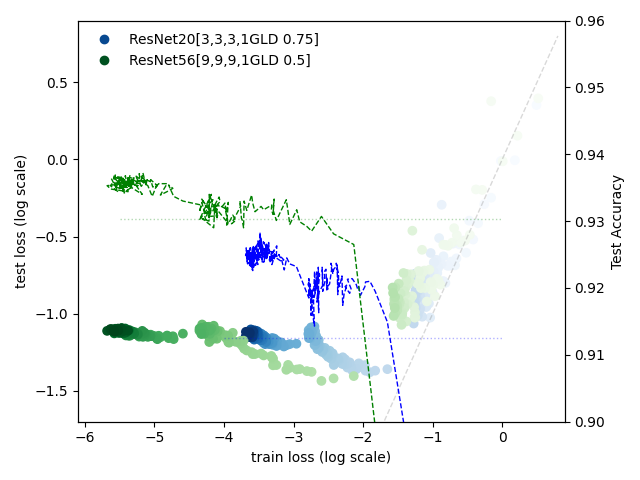}
    \end{subfigure}
\caption{\textbf{Structural generalization layer with drop-out (GLD) on CIFAR10:} Both architectures ResNet20 (blue) and ResNet56 (green) based on original ResNet architectures referenced in the labels have an additional GLD layer between encoder and decoder part of the architecture. For this additional cost they outperformed original ResNet models by quite a margin (see Table\ref{tab:ResNet_comparison} for details) compared to the original paper \cite{he2016deep} - color-coded dotted lines indicate the reference accuracy levels reached therein. A color-coding of epochs goes from lighter early epochs to darker later ones.}
\label{fig:gld_performance_1}
\end{figure*}

The best results of the tested models based on ResNet architecture with additional GLD layers are depicted in Fig.\ref{fig:gld_performance_thebestof}.

\begin{figure*}[htb]
    \centering
    \begin{subfigure}[t]{.85\linewidth}
    \centering
      \includegraphics[width=0.95\linewidth]{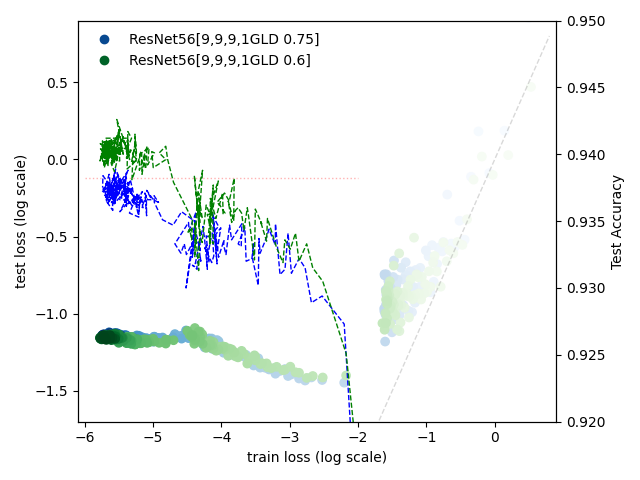}
    \end{subfigure}
\caption{\textbf{Hyperparameter adjusted structural generalization layer with drop-out (GLD) on CIFAR10:} This figure presents the best results achieved in the experiments on CIFAR10 and ResNets. The ResNet56 with GLD layer and a dropout rate of 0.6 outperformed all models from original ResNet paper \cite{he2016deep} including the deepest ResNet1202 model (reference line in red), see Table \ref{tab:ResNet_comparison} for details. The color- coding used is the same as in Fig. \ref{fig:gld_performance_1}}
\label{fig:gld_performance_thebestof}
\end{figure*}

Summarized in the Table\ref{tab:ResNet_comparison} overall results of ResNet architectures of varying depth on the CIFAR10 dataset are presented for comparison. Achieved test errors of the original paper in $4^{th}$ column and implementation of the same by \cite{Idelbayev18a} in the last column are compared to models using \textit{generalization layer} - the dropout variant, denoted GLD. The table is sorted according to the achieved test errors (the last column) in descending order. The rows in bold demonstrate results of models with generalization layer based on ResNet20 and ResNet56, in the $3^{rd}$ and two last rows respectively. As the last columns show ResNet20 GLD outperformed larger model ResNet32 and ResNet56 GLD gained the best results of all the models including the ResNet1202 with 1202 layers and 19.4 million parameters.

\begin{table}[]
    \centering
    \begin{tabular}{cccccc}\toprule
         &  Name & Layers & Params & Test err (orig.) & Test err $\downarrow$\\ \midrule
         & ResNet20	& 20&	0.27M&	8.75\%&	8.27\% \\
         & ResNet32	& 32&	0.46M&	7.51\%&	7.37\% \\ \midrule
         & \textbf{ResNet20 GLD 0.75} & \textbf{22}&\textbf{0.3M}&	--\%&\textbf{7.17\%} \\ \midrule
         & ResNet44	& 44&	0.66M&	7.17\%&	6.90\% \\
         & ResNet56	& 56&	0.85M&	6.97\%&	6.61\% \\
         & ResNet110& 110&	1.7M&	6.43\%&	6.32\% \\
         & ResNet1202& 1202&19.4M&	7.93\%&	6.18\% \\ \midrule
         & \textbf{ResNet56 GLD 0.5} & \textbf{58}&\textbf{0.93M}&	--\%&	\textbf{6.11\%} \\ 
         & \textbf{ResNet56 GLD 0.6} & \textbf{58}&\textbf{0.93M}&	--\%&	\textbf{5.74\%} \\ \bottomrule
         \end{tabular}
\caption{ResNet with and without \textbf{Generalization Layer with Drop-Out (GLD)}. GLD is placed after the encoder block of ResNet as opposed to GD with Skip connection experiments earlier, testing the effect of an invariance principle to arbitrary coordinate representations outlined in the  Discussion section. Outstanding results of GLD presented here are in support of placement of the generalization layer beyond the layers that are supposed to generalize well, e.g, encoder block, in line with this invariance principle.}
\label{tab:ResNet_comparison}
\end{table}
\section{Discussion and Future Work} \label{appsec:remarks}
As argued in \cite{zhang2016understanding} the norm of weights does not necessary captures good generalization. They show in particular that generalization in ReLu networks is invariant along hyper planes corresponding to reciprocal rescaling in and out side of a nonlinearity by some constant $\beta$ and $\frac{1}{\beta}$ respectively. If $\beta$ is absorbed into weights a norm along such hyperplane gets arbitrarily large despite it represents the same function and thus has the same generalization properties. As can be seen such a rescaling has no effect on a bound \eqref{col_2:pathproducts_upperb} of Corollary \ref{corollary_3} because $\beta$'s would cancel out along the path products involved. Regularizing a path products as in our method is more subtle than regularizing norms as follows from geometric vs. arithmetic mean or more general Jensen's inequality. More over $f$ with regularized spectral products as in Proposition \eqref{theorem:generalization_theorem} share many properties with low spectral rank used to characterize simple functions in \cite{huh2021low}. 

On the other hand, because Corollary \ref{corollary_3} is based on the upper bound that may be loose, there are differences in effectiveness across different versions of \textit{generalization layer}, as shown in Experiments i.e. GL vs. GLD in this paper. We believe many other forms of "generalization layer(s)" are to be explored as Proposition 3.1. only requires to control (ideally but not necessarily) all backpropagation paths $\mathcal{BP}$ no matter architecture or the way it is done so. Also, in experiments with dropout variant of generalization layer (GLD), varying results of different drop-out hyperparameter presented in Table \ref{tab:ResNet_comparison} suggest that a hyperparameter optimization may bring about even further improvements. All these suggestions are left for future work.

\textbf{On the depth and the width of the model}
The \textit{Generalization layer} keeps the data independent part of the the upper bound \eqref{col_2:pathproducts_upperb} in the Corollary \ref{corollary_3} small. Nevertheless the $\sigma_{\psi}(x)$ part, defined as the largest eigenvalue of the outer layer inner product of the tangent space at $TF(f(\bm{x}))$, plays its role too. The Corollary shows that the "generalization" benefits from flatter FIM, i.e. small $\sigma_{\psi}(x)$. We argue this is easier met by a high capacity model that is capable to reach more optima, because of its flexibility, and converge to a "good" one. Same as Remark [Towards Global Generalization] in section \ref{remark:global_generalization_in_manifold_of_distributions} this also suggests the generalization is conditioned by model's capacity. In particular, that means the model has to be deep enough because eigenvalues of layers are kept low by \ref{corollary_3} enforcing generalization and limiting a capacity due to the width. Proving this conjecture is left as future work.

In addition, Proposition 3.1 also provides a view on the role of the depth of the network with regards to a generalization seen as smoothness of the transformation $f$. An element $L_p(x)$ is a product of pointwise derivatives of 1-Lipschitz activation functions along the path. Since the common activation functions like ReLu, $tanh$, and their variants are 1-Lipschitz, having a pointwise derivative in the range $[0,1]$, the deeper the network the smaller $L_p(x)$ gets. Hence the 'simpler' and better generalizing $f$ is obtained.

As opposed to the depth of the model the effect of a width of layers is more involved. On one hand, it contributes with more paths to the sum in Eq.\ref{col_2:pathproducts} on the other hand weights, if initialized randomly in a common way, e.g. 'He' or 'Xavier', \cite{he2015delving, glorot2010understanding}, have zero mean and variance that corrects for a number of units. So the path products involved are of both signs\footnote{despite the overall Eq.\ref{eq:prop31_zeta} is always non-negative due to the positive semi-definiteness of $\partial_{\kappa}\partial_{\lambda}\psi(f(\bm{x}))$} and therefore the sum of the products is not guaranteed to grow over the limit even long after the initialization. Moreover, the width contributes to the capacity of the network, which is essential for generalization as argued above.

\subsubsection*{Popular regularizers in the light of Corollary 3.0.1} \label{appsec:common_regularizers}
As noted in the Introduction of the paper there are many regularizers at hand to be combined with SGD training and that works provably and empirically well. Next we relate most common regularization techniques to our results and show that they are in line and supportive each other.

\textbf{Weight decay} \cite{goodfellow2016deep} As outlined in the Discussion section in the main body of the paper $l_2$ norm regularization of weights could be linked to the upper bound Eq.\eqref{col_2:pathproducts_upperb} by trace and operator norm inequality (see Supplementary Material) on layer weight matrices, i.e. $tr({W^{(l)}}^T W^{(l)})$ from the proof of Corollary \ref{app:corollary_3}. Thus keeping $l_2$ norm of weights small keeps the upper bound of the layer's largest eigenvalue small and hence contributes to smaller Eq.\eqref{col_2:pathproducts_upperb}. Note however that this bound is rather loose in general and secondly, as noted earlier a weight decay is scale dependent and as such may rule out the optima of a large $l_2$ norm that generalize well according to our results and in line with \cite{zhang2016understanding}.
Nevertheless a stratified or selective weight decay applied only on the "generalization layer" may be just another way of keeping path or eigenvalue products low and therefore beneficial for generalization of the model. We leave this and other alternatives of the "generalization layer" design for a future work as well as cyclical learning rate \cite{smith2017cyclical} performing especially well in case of resnet and other deep architectures \cite{smith2018disciplined}.

%use of Gershgorin disks, \cite{brakken2007gershgorin}, that provide a range for eigenvalues of the matrix formed by its off diagonal row sums for instance.

\textbf{Batch normalization (BN)} Authors of BN in \cite{ioffe2015batch}, section 3.3 and 3.4, elaborate on the regularization and effect BN has on weights. BN arguably makes training more resilient to the parameter scale. In particular they argue that back-propagation through the layer is unaffected by the scale of its parameters and more over, larger weights lead to smaller gradients due to larger variance in nodes and thus also the denominator in the BN. In the effect BN stabilize the parameter growth. Referring to the Corollary \ref{corollary_3} it stabilizes the growth of eigenvalues in layers and in $\prod\limits_{l \in L_{f}} \sigma^2_{l}$ and therefore slows down the rate of convergence towards "complex" functions that overfit which seems to be the inevitable course of actions of vanilla SGD as we have shown in Fig.2 of the paper.

\textbf{Drop-out}
See. \cite{srivastava2014dropout}. The authors of batch normalization in their work \cite{ioffe2015batch}, see above, suggest based on experiments that batch normalization reduces, partially or completely, the need of drop-out, suggesting similar effect on the training. The same arguments as for BN above would apply here. Indeed dropput by multiplying, a random or deterministic, subset of layers' outputs by zero \cite{srivastava2014dropout, goodfellow2016deep} and thus deactivating weights leading to those units from gradient update at the given step slows the growth of the weight parameters similarly to BN. Or alternatively using \cite{hinton2012improving} to approximate dropout effect as the full model but with outgoing weights of node $i$ multiplied by probability of including the unit $i$. Applied to the path products Eq. \eqref{col_2:pathproducts} all (independent by method design) probabilities multiply leading to the regularizing effect of the dropout. The deeper the network the larger the effect by this approximation.

\textbf{Early stopping} Combined with a common random initialization of weights around zero, i.e. of \cite{he2015delving, glorot2010understanding} zero mean and variance that corrects for number of units of the network weights are gradually updated over the course of the learning as also shown in the experiment of Fig.1. Results of Fig.2 also suggests that vanilla SGD leads towards complex over fitting map $f$ characterized by large weights. Early stopping is a effective and robust way to stop weights along the way, \cite{li2020gradient}, and earlier it stops the smaller the upper bound of max eigenvalues in Eq. \eqref{col_2:pathproducts_upperb} is obtained hence producing 'simpler' functions.

\subsubsection*{Coordinate Representation Invariance Principle}
Another interesting topic for future work is to explore the number and placement of generalization layer in the original architecture. Experiments, Fig.\ref{fig2:regularization_by_longskips} vs. Fig.\ref{fig:gld_performance_1} suggest that better results are achieved when generalization layer is placed between encoder and decoder block of ResNet. Moreover, significantly better generalization properties of GLD models raise the research question of why that is so when according to Proposition 3.1 placements should not matter.

It can be motivated by following invariance principle\footnote{for the idea of invariance see \cite{amari2016information, chentsov1982statistical} where it is applied on the transformation of variables, however. That is in a different context.} with regards to an arbitrarily chosen input representation.

Recall that in the "coordinate representation of data manifold" view the input representation is rather arbitrary according to \cite{hauser2018principles}. So while Proposition 3.1 addresses generalization properties w.r.t. inputs $\bm{x}$ or seen from a forward pass perspective w.r.t. all the layers before the GL because it regularizes corresponding path products, all the layers following the GL one are unregularized and (may) cause an over-fitting. It follows from the application of Proposition 3.1 on any layer placed after the generalization layer and considering it a new input representation of the shallower model.

\section{Conclusions}
This paper develops a novel approach to the generalization of deep learning, a unifying geometrical perspective, the Learning in the manifold of distributions. It encompasses both classification and prediction neural networks models. Devised theory and Corollary \ref{corollary_3} is used to design a new method called "generalization layer" that is embedded into the architecture of the model as a structural regularizer. Further, the developed framework suggests that in deep enough models, as opposed to shallow models, such a regularizer enables both, extreme accuracy and generalization, to be reached.

In the experimental section we empirically verify that even simple setups, i.e. an imputing an extra "generalization layer" and keeping its eigenvalues low, improves the generalization. Another variant of structural regularizer based on \textit{generalization layer} concept using drop-out is developed. To confirm many ways to implement the generalization by structure are possible. And more importantly to test the role the placement of the generalization layer in the architecture plays. 

On that note the outstanding results on the CIFAR10 dataset corroborate the theory as well as the validity of invariance to coordinate representations principle from the discussion section. 

In conclusion, to impose the invariance of the model to the arbitrary coordinate representation of data manifold the generalization layer has to be placed \textit{after} all the layers that are to generalize well. The experiments with generalizing layer with drop-out placed after encoder block reported in Table \ref{tab:ResNet_comparison} confirmed these conclusions and with only 56 layers it outperforms by a margin the deepest 1202 layers ResNet model from original paper \cite{he2016deep}.

Further we discuss common regularization techniques that are placed into a perspective of this paper and are shown to be in line with its theory. Overall we believe that "generalization by design" provides both theoretical and methodological novelties and we hope to inspire a new line of research leading to better generalizing architectures.
\bibliography{listof}
\bibliographystyle{unsrtnat}
\clearpage
%\appendix
\section{Supplementary Material}
\appendix

\section{Proof of Proposition 3.1} \label{sec:appendix_proof}
This section is dedicated to proof of Proposition 3.1 from the main body of the paper and its consequences. It includes only a necessary minimum of definitions for brevity and we refer to an excellent manuscript \cite{hauser2018principles} or other resources where needed. The section concludes with general remarks on wider consequences of statements proven.

From now on we use upper indices to denote coordinates while lower ones index vectors/tensors. Also the Einstein summation is used whenever pair of indexes appears in the equation and it is not stated otherwise.

Without loss of generality, with a note that popular neural network architecture with ReLu activations can be seen as a limit of models using softmax (or softplus) activations \cite{sharma2017activation}, consider $f$ being a smooth $\mathcal{C}_2$ a non-invertible map $f:I \xrightarrow{} F$ between input manifold $I$ with a coordinate system denoted $\xi^i$ and output manifold $F$ with a coordinate system $\theta^{\kappa}$. As such it defines push-forward operator $f^*$ that acts on tangent spaces: $f^*:TI \xrightarrow{} TF$ and it can be viewed as a generalized coordinate free derivative. We leave the details out and refer interested reader to \cite{hauser2018principles} for more.

Following the main document we'd like to link generalization of the network $f$ to its structure given by a composition of layer to layer maps, defined by layer dimension, activation function(s) and weights $\mathrm{W}^{(l)}$ of layer $l$. 

Because an input space is an Euclidean space with an ortho-normal basis it is the same as its tangent space. Let's denote $g_{i,j}(\xi)=\langle \bm{e}_i, \bm{e}_j \rangle$ an inner product on the input space. It is a constant identity matrix in our case, i.e. $g_{i,j}(\xi)=\delta_{i,j}$ and denoted further as $g_{i,j}$. Also as of this point further the manifold $I$ coincides with input data layer $X$. We'll use both notations interchangeably from now on in this section.

Similarly for an output layer (a probabilistic manifold of probability distributions defined by a choice of Bregmann loss function corresponding to a cummulant function $\psi$) we have its metric tensor defined as 
\begin{align}
    g_{\kappa,\lambda}(\theta)=E[\partial_{\kappa}\log p(\tilde{\bm{y}},\theta) \partial_{\lambda}\log p(\tilde{\bm{y}},\theta)] \text{ (FIM) }=\partial_{\kappa}\partial_{\lambda}\psi(\theta) \label{appeq:FIM_as_hessian}
\end{align}
where $\tilde{\bm{y}}=\nabla \varphi(\bm{y})$ is a random variable of the output layer Exponential family distribution derived from a dual Bregman divergence using \eqref{def:mean_natural_duality}, $\varphi$ and $\psi$ being convex conjugates, see Section \ref{sec:appendix_bregman} or \cite{banerjee2005clustering} and where the second equation is a consequence of a dual flatness of probabilistic manifold, \cite{amari2016information}, Theorem 2.1. therein. 

An infinitesimally small line element $d\theta_{\kappa}$ in an ouput tangent space $TF$ relates to an input $dx^i$ vector by Jacobian $J^{\kappa}_i$
\begin{align}
    d\theta_{\kappa}&=J^{\kappa}_i dx^i \label{appeq:small_line_element} \\
    J^{\kappa}_i&= \frac{\partial f^{\kappa}}{\partial \xi^i} (\bm{x}) \label{eq:jacobian}
\end{align}
where Einstein summation over index $i$ is used.

A shape on the output layer is given by its metric tensor, i.e., Fisher information matrix (FIM) of the probabilistic model induced by the chosen Bregman divergence, see Section \ref{sec:appendix_bregman} or \cite{amari2016information}. At a given point $\tilde{\bm{y}}$ it characterized by positive semidefinite matrix $H$ (hessian) its curvature can be analyzed by exploring its eigenvalues. In particular the flatter the landscape w.r.t. outputs of network $f(\bm{x})$, that play the role of natural parameters of induced Exponential family see Appendix \ref{sec:appendix_bregman}, the smaller the change in likelihood as a function of $d\theta$ in its neighbourhood $f(\bm{x}) + d\theta, d\theta \in TF(f(\bm{x}))$ and thus the better it generalizes.

\textbf{Back-propagated Inner Product of the Manifold of Distributions $\zeta(x)$} \\
Nevertheless we'd like to know how the log likelihood changes with regards to an input layer and how does it depend on the model parameters. We'll follow the "flatness" of the loss landscape as measure of generalization above. But we measure the curvature of the "back-propagated Hessian" ("pulled back"\footnote{\cite{hauser2018principles} show that it is equivalent to pulling back the (output layer) frame bundle of a data manifold along the map $f$ and such it is well defined operation. We omit formal definitions for brevity and interested reader is kindly referred to \cite{hauser2018principles} and \cite{amari2016information}} metric) in the input layer instead. It is defined by its elements $\zeta(\bm{x}) := \{\zeta^{i,j}(\bm{x})\}$ as follows:
\begin{align}
     \zeta^{i,j}(\bm{x}):=J_i^{\kappa}(f(\bm{x}))\partial_{\kappa}\partial_{\lambda}\psi(f(\bm{x}))J_j^{\lambda}(f(\bm{x})) \label{def:backpropagated_hessian}
\end{align}
whenever $\bm{x} \in I$, and $f(\bm{x}) \in F$. By the convexity of $\psi$ that follows from the choice of Bregman divergence as a loss the $\zeta(\bm{x})$ defines a real positive semidefinite and symmetric quadratic form on a finite dimensional input vector space $\mathrm{R}^{dim(\mathrm{X})}$.

This "back-propagated Hessian" carries the information of functional characteristics of a map $f$\footnote{similar to Riemann-Christoffel (RC) curvature tensor that captures the change of vector transported back to the origin along the closed loop curve. This round-the-world transport changes the original vector depending on the curvature of the manifold, see \cite{amari2016information}} along which we pull the outer Hessian back to an input layer and we can use it to capture the degree of the "generalization" of the map $f$ at datum $\bm{x}$. 

We again emphasize that here we for brevity leave out the details of defining "pull-back" on frame bundles and refer to \cite{hauser2018principles} where it is properly done ensuring that this "back-propagation" is well defined.

\textbf{Max eigenvalue of $\zeta(x)$ and well generalizing functions}\\
In a case $f$ was (invertible) coordinate transformation we'd have two metric tensors related by standard Jacobian relation:
\begin{align}
    g_{i,j}=J_i^{\kappa}(\theta)J_j^{\lambda}(\theta)g_{\kappa,\lambda}(\theta) \label{def:jacobian}
\end{align}
where $J_i^{\kappa}$ is a Jacobian matrix which depends on $\theta$ in general. Local distances between input and output manifolds would relate as $ds^2=g_{i,j}dx^idx^j=g_{\kappa,\lambda}(\theta)d\theta_{\kappa}d\theta_{\lambda}$.

Since $f$ is \textbf{not} a change of coordinates in general we cannot use the last two expressions to relate distances in tangent spaces in a usual way.

Instead we make use of a dual flatness of the output layer with an inner product $g_{\kappa,\lambda}(\bm{\theta})$ induced by a choice of Bregman loss. A small local distance $ds^2$ in the output layer can be written in output layer coordinates as well as in input layer ones as follows (using Einstein summation):
\begin{align}
    ds^2=g_{\kappa,\lambda}(\bm{\theta})d\theta^{\kappa}d\theta^{\lambda}=\partial_{\kappa}\partial_{\lambda}\psi(\bm{\theta})J^{\kappa}_i dx^iJ^{\lambda}_j dx^j \label{proofeq:ds2_1}
\end{align}
where the first equality follows from output layer being metric space and we plugged in Eq.\eqref{appeq:small_line_element} and Eq.\eqref{appeq:FIM_as_hessian} to obtain the second equality. More over, following the notation of \cite{amari2016information}, $\theta^{\lambda}$ denotes a $\kappa^{th}$ coordinate curve (in our case of a dually flat space it is a geodesics) and a tangent vector $\bm{e}_{\lambda}$ is defined as a partial derivative operator $\bm{e}_{\lambda}:=\partial_{\lambda}=\dfrac{\partial}{\partial\theta^{\lambda}}$ that operates on a differentiable function and it gives its derivative in the direction of a coordinate curve $\theta^{\lambda}$. 

Formula $\partial_{\kappa}\partial_{\lambda}\psi(\bm{\theta})$ is understood accordingly as a composition of two derivative operators, acting on differentiable functions $\psi(\bm{\theta})$ and $\partial_{\lambda}\psi(\bm{\theta})$ consequently, along the coordinate curves $\theta^{\lambda}$ and $\theta^{\kappa}$, see  \cite{amari2016information}, Section 5.

By of convexity of $\psi$ we have that $\zeta(\bm{x})$ is a positive semidefinite symmetric real matrix of a dimension of the input layer, $dim(\mathrm{X})$, for all $\bm{x} \in I$ as defined in Eq.\ref{def:backpropagated_hessian}. As such its eigenvalues are all nonegative. Let's denote $\sigma_{max}(\bm{x})$ its largest eigenvalue evaluated at $\bm{\theta} = f(\bm{x})$ as follows:

\begin{definition}[$\sigma_{max}(\bm{x})$]\label{appdef:sigma_max_label} Given a datum $\bm{x} \in \mathrm{R}^{dim(\mathrm{X})}, dim(\mathrm{X}) < \infty$ and the smooth map $f:\mathrm{X} \xrightarrow{} \mathrm{\tilde{Y}}$ the $\sigma_{max}(\bm{x})$ is defined as the largest positive eigenvalue positive semidefinite matrix $\zeta(\bm{x})$ defined by Eq.\ref{def:backpropagated_hessian}.
\begin{align}
    \sigma_{max}(\bm{x}):=\max_{\sigma}\{\sigma: \det(\sigma I -\zeta(\bm{x}))=0\} \label{appdef:sigma_max}
\end{align}
where $\det()$ and $I$ denote a determinant and the identity matrix respectively.
\end{definition}
Note that this is a standard definition and we restated it only to capture and emphasize its dependence on input datum $\bm{x}$. 

\begin{remark} [Generalization driven by Jacobian of $f$]
To assess the degree of "flatness" of the loss landscape with regards to a neighbourhood of the given input $\bm{x}$ the largest eigenvalue of pulled-back metric $\sigma_{max}$ can be used as it defines the curvature of this neighbourhood as depicted in Fig. \ref{fig:pulled_back_hessian_curvature}. The larger the $\sigma_{max}(\bm{x})$ the more curved is the loss landscape (over all directions in $I$) in the neighbourhood of the input $\bm{x}$. Note that this curvature $\sigma_{max}(\bm{x})$ is given by two factors: 1.) curvature of the probability output manifold, i.e. how close prediction $f(\bm{x})$ is to the ML estimate of the induced probabilistic model as well as 2.) the derivatives of $f(\bm{x})$ w.r.t. $\bm{x}$ that capture how $f$ changes w.r.t. its inputs. This is an essential concept when generalization instead of Hessian relies on Jacobians (first derivatives of $f$ or consequently loss if output layer is taken into consideration). See. also section 2 for the idea of generalization within learning in the manifold of distributions concept. 
\end{remark}

\begin{figure*}[t]
\centering
  \includegraphics[width=.95\linewidth]{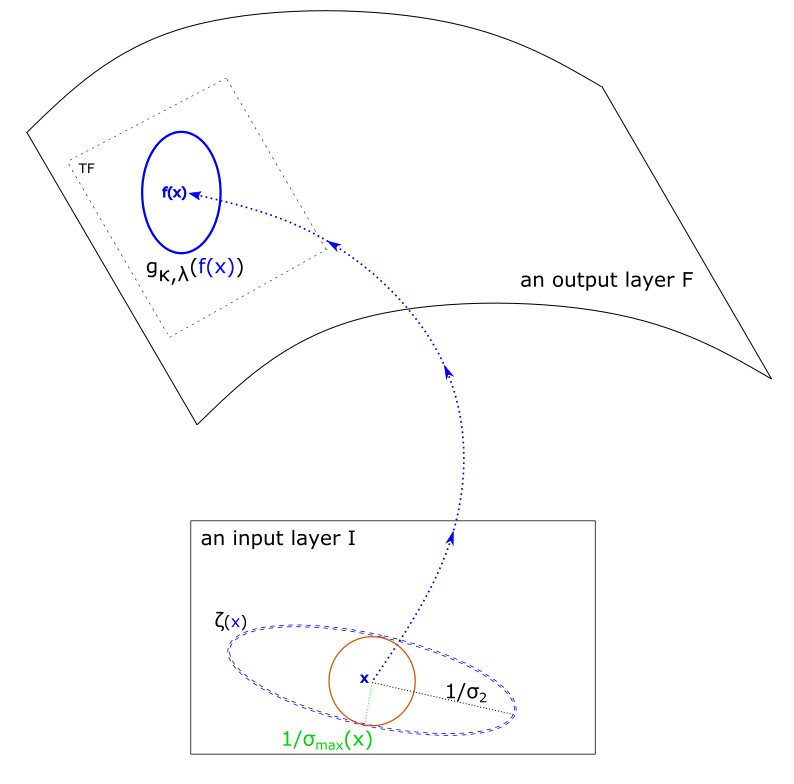}
    \caption{\textbf{Pulled-back metric $\zeta (x)$}}
\label{fig:pulled_back_hessian_curvature}
\end{figure*}

Further let's note that the definition \eqref{appdef:sigma_max_label} is 'local' and 'differential' in the sense it depends on a data point $x$ of the data manifold $M$ and valid in an infinitesimal neighbourhood of $x$ (and therefor $f(x)$ from its assumed smoothness) by the use of differential geometry tools. This 'locality' is a given by the fact that the outer layer inner product (FIM) smoothly varies over coordinates of an output layer as a consequence of its dual flatness and more over a map $f$, which could be thought of as a "coordinate transformation" between tangent spaces, also changes non-linearly with $x$ in general.

Next we restate the Proposition 3.1 a proof of which is now straightforward consequence of the preceding text:
\begin{proposition}[Proposition 3.1]\label{appproposition:generalization_prop}
In the context of the above, assuming activation functions used in architecture of $f$ are 1-Lipschitz (for definition see \ref{appdef:lipschitz}) a pull-back metric from output layer into an input Euclidean manifold around datum $x$ is up to a constant defined by following positive semidefinite matrix:
\begin{align}
    \zeta(x) &= P(x) {\zeta_{\psi}(\bm{x})} P^T(x) \label{appeq:prop31_zeta}
\end{align} 
where $P(x)$ is a real matrix $\{p_{i,j}\}$ with elements:
\begin{align}
    p_{i,j} &:= \sum\limits_{p \in \mathcal{BP}} {}^i_j\pi_p L_p(x), \label{appcol_2:pathproducts}
\end{align}
$\zeta_{\psi}(x):=\partial_{\lambda}\partial_{\lambda}\psi(f(\bm{x}))$ is fully determined by chosen Bregman divergence as a loss. Further $L_p(x)$ is positive real function formed from products of activation functions derivatives along the path $p$ such that $0 \leq L_p(x) \leq 1$ and where $\mathcal{BP}$ is a set of all "back-propagation" paths connecting any input layer node to an output layer node through the network $f$ such that each layer has exactly one node present in the path. Then ${}^i_j\pi_p = \prod\limits_{\{l:w_l \in p\}} w_l$ is a product of all weights from input node $i$ to output node $j$ along the path $p$.
\end{proposition}
\begin{proof}
The first statement of Eq.\ref{appeq:prop31_zeta} is a rewritten definition of $\zeta_{\psi}(x)$, Eq.\ref{def:backpropagated_hessian}, into a matrix form where we take $P(x):=J(\bm{x})$ to be Jacobian of $f$ with elements $\{j_{i,\kappa}\}:=J^{\kappa}_i= \frac{\partial f^{\kappa}}{\partial \xi^i} (\bm{x})$ defined in Eq.\ref{def:jacobian} that has $dim(X)\times\dim(TF)$) dimensionality.

The second statement, Eq.\ref{appcol_2:pathproducts} follows from $P(x):=J(\bm{x})$ above and writing product of layer weight matrices in a form of sum together with assumption on $f$ being 1-Lipschitz, defined in \ref{appdef:lipschitz}.
\end{proof}

\begin{corollary}[Corollary 3.0.1, Spectral products, Informal] \label{app:corollary_3} Under conditions of Corollary \eqref{appproposition:generalization_prop} the largest eigenvalue of $\zeta(x)$ can be bounded from above by a following product of eigenvalues:
\begin{align}
    \sigma_{\psi}(x) C\prod\limits_{l \in L_{f}} \sigma^2_{l} , \text{ and $C$ is a positive constant} \label{appcol_2:pathproducts_upperb}
\end{align}
where $L_{f}$ denotes all layers of $f$ and $\sigma_{l}$ denotes the largest eigenvalue of matrix $W^{(l)}$ comprising the weights of the layer $l$ and similarly $\sigma_{\psi}(x)$ denotes the largest eigenvalue of the positive semidefinite outer layer metric tensor $\zeta_{\psi}(x)$ from Eq.\ref{appcol_2:pathproducts}.
\end{corollary}
\begin{proof}
We can rewrite $\zeta(x) = P(x) {\zeta_{\psi}(\bm{x})} P^T(x)$ from \ref{appproposition:generalization_prop} in a gradient back-propagation style as a product of weight matrices and activation derivatives over layers:
\begin{align}
    &\zeta(x) = P(x) {\zeta_{\psi}(\bm{x})} P^T(x) \notag \\
    =& W^{(0)}\odot c^{'(0)}W^{(1)}\odot c^{'(1)}\dots W^{(L)}\odot c^{'(L)} {\zeta_{\psi}(\bm{x})} {W^{(L)}\odot c^{'(L)}}^T\dots {W^{(1)}\odot c^{'(1)}}^T {W^{(0)}\odot c^{'(0)}}^T  \label{eq:back_prop_style}
\end{align}
where $W^{(l)}$ denotes weight matrix of the layer $l, l \in \{0,\dots ,L\}$ and $c^{'(l)}$ is a vector of layer nodes activation functions's derivative and $\odot$ operation between matrix $M$ and vector $v$ of suitable dimension is defined as $M\odot v := M \cdot diag(v)$.

Next we will make use of the following well known matrix relations (see for instance \cite{bhatia1997a, gantmakher1959theory}):
\begin{align}
    tr(ABC)&=tr(BCA) \tag{'cyclic property of a trace'} \label{def:cyclic_trace} \\
    & \text{ for any three real matrices $A$ and $B$ such that products and traces involved are defined } \notag \\
    tr(AB)& < tr(A)tr(B) \text{ for $A$ and $B$ real and symmetric} \label{def:trace_trtr} \\
    \frac{tr(A)}{dim(A)}& \leq \sigma_{max}(A) \leq tr(A) \text{ for $A$ real symmetric} \label{ineq:tr_sigma}
\end{align}

Applying \eqref{ineq:tr_sigma} on \eqref{eq:back_prop_style} and further by applying \ref{def:cyclic_trace} to reorder the products of Eq.\ref{ineq:tr_sigma} and \ref{def:trace_trtr} we get:
\begin{align}
    \sigma_{max}(\bm{x}) &\leq tr(\zeta(x)) \leq dim(F)\sigma_{\psi}(x)\prod\limits_{l \in \{0, \dots, L\}} tr({W^{(l)}}^T W^{(l)})tr(diag(c'^{(l)})^2) \notag \\
    &\leq dim(F) C \sigma_{\psi}(x) \prod\limits_{l \in L_{f}} \sigma^2_{l}
\end{align}
where the last inequality follows by applying the left inequality in Eq.\ref{ineq:tr_sigma} and a 1-Lipschitz property (see \ref{appdef:lipschitz} for definition) of the activation functions by assumption. The constant $C$ comes from lower bound of Eq.\ref{ineq:tr_sigma} and is a product of dimensions of layers. By absorbing $dim(F)$ into $C$ the statement follows.
\end{proof}

\section{Background on Bregman divergences, Exponential family and Notation} \label{sec:appendix_bregman}
\noindent Let a neural network $f:\mathbb{R}^b \times \Upsilon \xrightarrow{} \mathbb{R}^d$ of $L$ layers be defined as the composition:
\begin{align}
    f(\bm{x},\overrightarrow{\bm{w}}) =&\varphi_L(\bm{W}^{(L)}) \circ \varphi_{L-1}(\bm{W}^{(L-1)}) \circ \dots \notag \\
    \circ &\varphi_1(\bm{W}^{(1)})(\bm{x}) \label{def:model}
\end{align}
where each vector function $\varphi_l(\bm{W}^{(l)})(\bm{v})=a_l\left(\bm{W}^{(l)}\bm{v}\right)$ is an activation function $a_l$ applied onto a result of matrix $\bm{W}^{(l)}$ and vector $\bm{v}$ product. We denoted collation of all network weights into a tensor as $\overrightarrow{\bm{w}} \in \Upsilon$.

For reasons to be revealed shortly we define loss function as Bregman divergence \cite{banerjee2005clustering} :
\begin{align}
    \mathcal{L}(\bm{z},\bm{y}) &= d_{\Phi}(\bm{z},\bm{y}) \notag \\
    &= \Phi(\bm{z}) - \Phi(\bm{y}) - \langle \bm{z}-\bm{y},\nabla_{y} \Phi(\bm{y})\rangle
    \label{def:bregman}
\end{align}
, where $\Phi: \mathbb{R}^{d} \xrightarrow{} \mathbb{R}$ is strictly convex function.

Overall, following the generalization framework of \cite{kawaguchi2017generalization} we aim to minimize $\ell(\overrightarrow{\bm{w}};\mathcal{D},f)$ given training dataset indexed by set $\mathcal{D}$ and hypothesis captured in composition of $f$:
\begin{align}
    \ell(\overrightarrow{\bm{w}};\mathcal{D},f)=\int\limits_{(\bm{x},\bm{y}) \in \mathcal{D}}  \mathcal{L}(f(\bm{x},\overrightarrow{\bm{w}}),\bm{y}) d\mathbb P(\bm{x},\bm{y}) \label{def:loss}
\end{align}

From now on we will abuse notation and use $\bm{w}$ instead of $\overrightarrow{\bm{w}}$ to denote either all or subset of weight(s) depending on the context. We will also use index instead of a function argument to denote dataset over which loss is evaluated, i.e. $\ell_{\mathcal{D}}(\bm{w})$ instead of $\ell(\overrightarrow{\bm{w}};\mathcal{D},f)$. And we refer to value of loss over batch $B_i \subset \mathcal{D}$ as $\ell_{B_i}(\bm{w})$.

Further in this paper we consider back-prop training of network $f$ using stochastic gradient descent (SGD) with the constant learning rate $\eta$ over mini-batch samples indexed by $B_i$
\begin{align}
    \bm{w}_{k+1}=&\bm{w}_k - \eta\nabla_{w} \ell_{B_i}(\bm{w}_k) \label{eq:sgd}
\end{align}
In case of mini-batch being whole dataset we may refer to it as (full) gradient descent (GD) throughout the text.

\subsection*{Usefull properties of Bregman divergence}
Minimizing square-loss, cross-entropy and in general negative-(log)likelihood (KL-divergence) and many other objectives can be suitably captured by choice of strictly convex, differentiable function $\Phi(\bm{x})$
\begin{itemize}
    \item square-loss: $\Phi(\bm{x}) := \frac{1}{2}\langle \bm{x},\bm{x} \rangle$
    \item KL-divergence: $\Phi(\bm{x}) := \sum\limits_{j=1}^{d}x_j\log_2x_j$ s.t. $\sum\limits_{j=1}^{d}x_j=1$ (this also covers use of cross-entropy loss for classification tasks)
\end{itemize} 
For reference see e.g. \cite{banerjee2005clustering, hiriart2012fundamentals}.
\\
Using the Bregman divergence as loss function allows us to derive general results for a wide range of losses, including both classification and prediction problems.

\noindent An important property of the Bregman divergence is that its derivative w.r.t the first argument at datum $(\bm{x}_s,\bm{y}_s)$ evaluates as
\begin{align}
    \nabla_{x}d_{\Phi}(x_s,y_s) = \nabla \Phi{(\bm{x_s})} - \nabla \Phi{(\bm{y_s})} \label{eq:bregman_derivative}
\end{align}
\noindent Further it can be shown that there exist an isomorphic dual space such that 
\begin{align}
    d_{\Phi}(\bm{z},\bm{y}) = d_{\Psi}(\nabla \Phi{(\bm{y})}, \nabla \Phi{(\bm{z})}) \label{eq:duality} \tag{duality}
\end{align}
where $\Psi$ is a convex conjugate to $\Phi$. For more details see \cite{hiriart2012fundamentals}. This notion will crucial in developing generalization error surrogate in the next section.

\subsection*{Mapping between Exponential families and Bregman divergence}
As presented in \cite{banerjee2005clustering}, Theorem 4, there exists a one-to-one mapping between the regular exp. family of distributions $p_{(\Psi,\bm{\theta}) (\bm{x})}$ generated by sufficient statistics, base measure and Bregman div. $d_{\Phi}(\bm{x},\bm{y})$, (\textit{informaly})
\begin{align}
    p_{(\Psi,\bm{\theta})} (\bm{\omega})=exp(-d_{\Phi}(T(\omega),\bm{\mu}))b_{\Phi}(T(\omega))     \label{def:Bregman_to_exp_family}
\end{align}
where the related exponential family has the following form 
\begin{align}
    p_{(\Psi,\bm{\theta})} (\omega)=\exp\left(\langle \bm{\theta},T(\omega) \rangle - \Psi(\bm{\theta})\right)dP_0(\omega)
    \label{def:exp_family}
\end{align} 
and 
\begin{align}
b_{\Phi}(T(\omega)) = \exp(\Phi(T(\omega))) \label{eq:b_omega} 
\end{align}
is uniquely determined given the base measure $P_0(\omega)$\footnote{Note that exp. family is defined with respect to some carrier measure. Then density $p_0$ correspond to Radon-Nykodym derivative $\frac{dP_0(\omega)}{d\lambda(\omega)}$ where $P_0$ is absolutely continuous w.r.t. the Lebesque or counting (carrier) measure $\lambda$ for continuous and discrete r.v. respectively in an alignment with \cite{Wainwright08graphicalmodels}, \cite{banerjee2005clustering}.}.
Further for a clearer notation and without loss of generality we assume embedding of inputs $\Omega$ into a real vector space of dimensionality $dim(X)$, i.e. $\Omega \subset \mathcal{B}(\mathbb{R}^{dim(X)})$, a $\sigma$-algebra of Borel sets on $\mathbb{R}^{dim(X)}$.
To make this explicit we replace $\omega$ in notation and let $\bm{x}$ denote the element of Borel sets $\sigma$-algebra on $\mathbb{R}^{dim(X)}$. 

Let \textit{mean} and \textit{natural} parameters be denoted $\bm{\mu}$ and $\bm{\theta}$ respectively. Since $\Phi$ and $\Psi$ are Legendre duals there are also following known properties, see. \cite{Wainwright08graphicalmodels}
\begin{align}
    &E_{\bm{\theta}}[T(\bm{x})] = \bm{\mu}(\bm{\theta}) \\
    &\nabla \Psi(\bm{\theta}) = \bm{\mu} , \nabla \Phi(\bm{\mu}) = \bm{\theta} \label{def:mean_natural_duality}
\end{align}
for $\bm{\mu} \in int(dom(\Phi))$ so that $\nabla \Phi$ exists.

\noindent The conjugate function can be expressed as $\Phi(\bm{\mu})=\langle \nabla \Phi(\bm{\mu}), \bm{\mu} \rangle - \Psi(\nabla \Phi(\bm{\mu}))$,\footnote{follows from definition of $\Psi(\bm{\theta}):=\sup_{\bm{\mu} \in dom(\Phi)}\langle \bm{\theta} , \bm{\mu} \rangle - \Phi(\bm{\mu}))$ and because the supremum is attained at $\bm{\theta}=\nabla \Phi(\bm{\mu})$. We skip technicalities in definitions for a supremum to be attainable for the sake of space and brevity, see \cite{Wainwright08graphicalmodels} for details.} and thus we can write log likelihood of $p_{(\Psi,\bm{\theta})}(\bm{t})$ from Eq.\eqref{def:exp_family} as
\begin{align}
    \langle \bm{t},\bm{\theta} \rangle - \Psi(\bm{\theta}) &= (\langle \bm{\mu} , \bm{\theta} \rangle - \Psi(\bm{\theta}))+\langle \bm{t}-\bm{\mu} ,\theta \rangle \notag \\
    &= \Phi(\bm{\mu})+\langle \bm{t}-\bm{\mu},\nabla \Phi(\bm{\mu}) \rangle  
\end{align}

\noindent Therefore for any $\bm{t} \in dom(\Phi)$ and $\bm{\mu} \in int(dom(\Phi))$ we can write:
\begin{align}
    \langle \bm{t},\bm{\theta} \rangle - \Psi(\bm{\theta}) -\Phi(\bm{t}) &= -d_{\Phi}(\bm{t},\bm{\mu}) \label{eq:bregman_to_expon_expargument}
\end{align}

\subsubsection*{Max Likelihood in Exp. family} \label{sec:MLE_lower_bound_exp_fam}
Assume r.v. $X$ follows distribution from exponential family w.r.t. to some base measure as defined in \eqref{def:exp_family}. A duality of $\Phi$ and $\Psi$ leads to so called Fenchel's inequality for $\Psi(\theta)$, for reference see \cite{Wainwright08graphicalmodels}, variational representation of cumulant function, Theorem 3.4:
\begin{align}
    0 \geq \langle \theta,\mu \rangle-\Psi(\theta)-\Phi(\mu) \label{eq:MLE_lower_bound_exp_fam}
\end{align}
, where $\theta \in \Theta$ belongs to natural parameter space and $\mu \in \mathcal{M}^{\circ}$ is from interior of mean value parameter space. Relating it to \eqref{eq:bregman_to_expon_expargument} we see the right hand part of \eqref{eq:MLE_lower_bound_exp_fam} is a negative Bregman divergence $-d_{\Phi}(\bm{\theta},\bm{\mu})$.

It is well known that by maximizing this lower bound, corresponding to $-d_{\Phi}(\bm{\theta},\bm{\mu})$, the inequality \eqref{eq:MLE_lower_bound_exp_fam} turns into equality if and only if the mean value parameters are equal to the observed moments: $E_{\Psi}[\bm{t}(\bm{x}_{\alpha})]=\hat{\mu}_{\alpha} = \bm{y}_{\alpha}$ for $\alpha \in \mathcal{I}(V)$, where ${I}(V)$ denotes index set of observed nodes (data points) $V$, \cite{Wainwright08graphicalmodels}. In such a case $\Phi(\hat{\mu})$ is a negative Shannon entropy of the distribution matching given moments $\hat{\mu}_{\alpha}$ and it is maximal among all such distributions \cite{Wainwright08graphicalmodels,hiriart2012fundamentals}.

\subsection*{Dually coupled Exponential family}
There is an intriguing property of the Bregman divergences stating that the Bregman divergence $d_{\Psi}$ equals to the Bregman divergence $d_{\Phi}$ on the dual space defined by gradient mapping $\nabla \Phi$, for details see \cite{bauschke2011convex, banerjee2004information}. 

Thus the "dual" Bregman divergence $d_{\Psi}(\nabla \Phi(\bm{y}), \nabla \Phi(f(\bm{x}|\bm{w})))$ defines a dually coupled exp. family $p^*_{\Phi,{f(\bm{x}|\bm{w})}}(\bm{y})$ over the dual space\footnote{To avoid confusion with other mean values and more we do not use \textit{mean value} and \text{natural} labels of the duals in our settings}.

In lieu of this duality there are two dually coupled parametrizations of the related exponential family \dots 
\begin{enumerate} \label{def:dual_interpretation_of_gen_err}
    \item \textit{(primal)} \label{def:primal} \dots defined by a cumulant function $\Psi$ and neural network $f(\bm{x}|\bm{w})$ being a parametrized function of sufficient statistics. Mean value parameters are given by targets $\bm{y}$, given by Eq. \eqref{def:Bregman_to_exp_family}.
    \item \textit{(dual)} \label{def:dual} \dots defined by a cumulant function $\Phi$ and sufficient statistics $\nabla \Phi(\bm{y})$. Neural network $f(\bm{x}|\bm{w})$ is a parametrized subspace of a natural parameter space of the family in this formulation. 
\end{enumerate}

Note that the gradient mapping to the dual space is fully determined by the choice of $\Phi$ (or equivalently $\Psi$). For instance in case of the square loss, i.e. $\Phi(\bm{x}) := \frac{1}{2}\langle \bm{x},\bm{x} \rangle$, the gradient mapping $\nabla \Phi(\bm{y})$ onto dual space is an identity map.

\begin{definition} \label{appdef:lipschitz}
A function $f$ such that $|f(x)-f(y)| <= C |x-y|$ for all $x$ and $y$, where $C$ is a constant independent of $x$ and $y$, is called $C-$Lipschitz function.
\end{definition}
For example, any continuous function with the first derivative bounded in absolute value by $B$ must be $B-$Lipschitz.
\end{document}